\documentclass[11pt]{article}

\usepackage{microtype}
\usepackage{graphicx}
\usepackage{subcaption}
\usepackage{booktabs}
\usepackage{bm}
\usepackage{thmtools}
\usepackage{amsmath}
\usepackage{amssymb}
\usepackage{amsthm}
\usepackage{mathtools}
\usepackage{colortbl}
\usepackage{titletoc}
\usepackage{etoolbox}
\usepackage{multirow}
\usepackage{bbm}
\usepackage{stmaryrd}
\usepackage{enumerate}
\usepackage{physics}
\usepackage{pifont}
\usepackage{xcolor}
\usepackage{natbib}
\usepackage[margin=1in]{geometry}

\usepackage{nicefrac}
\usepackage{wrapfig}
\usepackage{xspace}
\usepackage{enumitem}
\usepackage{float}
\usepackage{siunitx}

\definecolor{lightgray}{gray}{0.92}
\definecolor{darkgreen}{rgb}{0,0.5,0}

\usepackage{algorithm}
\usepackage{algpseudocode}

\algrenewcommand\algorithmicrequire{\textbf{Input:}}
\algrenewcommand\algorithmicensure{\textbf{Output:}}

\DeclareCaptionFont{algorithmfont}{\fontsize{10}{11}\selectfont}
\newsavebox{\diffnashalgorithmbox}
\BeforeBeginEnvironment{algorithmic}{%
  \begin{lrbox}{\diffnashalgorithmbox}%
  \begin{minipage}{\linewidth}%
}
\AfterEndEnvironment{algorithmic}{%
  \end{minipage}%
  \end{lrbox}%
  \noindent
  \ifdim\dimexpr\ht\diffnashalgorithmbox+\dp\diffnashalgorithmbox\relax
    >0.88\textheight
    \resizebox*{!}{0.88\textheight}{\usebox{\diffnashalgorithmbox}}%
  \else
    \usebox{\diffnashalgorithmbox}%
  \fi
}

\AtBeginEnvironment{algorithmic}{%
  \fontsize{10}{11}\selectfont
  \setlength{\abovedisplayskip}{4pt plus 1pt minus 1pt}%
  \setlength{\belowdisplayskip}{4pt plus 1pt minus 1pt}%
  \setlength{\abovedisplayshortskip}{0pt plus 3pt}%
  \setlength{\belowdisplayshortskip}{4pt plus 3pt minus 3pt}%
}

\usepackage{hyperref}
\hypersetup{
    colorlinks=true,
    linkcolor=red,
    citecolor=darkgreen,
    filecolor=magenta,
    urlcolor=blue
}

\usepackage[capitalize,noabbrev]{cleveref}

\usepackage{thm-restate}

\theoremstyle{plain}

\newtheorem{proposition}{Proposition}
\newtheorem{lemma}{Lemma}

\theoremstyle{definition}
\newtheorem{definition}{Definition}

\theoremstyle{remark}

\usepackage{tikz}
\usetikzlibrary{
  calc, graphs, graphs.standard, shapes, arrows, arrows.meta,
  positioning, decorations.pathreplacing, decorations.markings,
  decorations.pathmorphing, fit, matrix, patterns, shapes.misc, tikzmark
}

\usepackage[textsize=tiny]{todonotes}

\newcommand{\ownpolicy}{\pi^i}
\newcommand{\opponentpolicy}{\pi^{-i}}
\newcommand{\ownaction}{\mathbf{a}^i}
\newcommand{\opponentaction}{\mathbf{a}^{-i}}
\newcommand{\opponentsample}{\mathbf{a}_{(\ell)}^{-i,k}}

\newcommand{\OMD}{\mathrm{OMD}}
\newcommand{\GWFP}{\mathrm{GWFP}}
\newcommand{\solver}{\mathrm{solver}}

\newcommand{\vaction}{\mathbf{a}}
\newcommand{\vsample}{\mathbf{z}}
\newcommand{\ELBO}{\operatorname{ELBO}}

\newcommand{\behavioriteration}{k^{(\ell)}}
\newcommand{\selfbatch}{\mathrm{B}_{\mathrm{self}}}

\newcommand{\const}{\mathrm{const}}
\newcommand{\isestimator}{\widehat{\Delta}^{k,\behavioriteration}}
\newcommand{\selfbehavioriteration}{k^{(m)}}
\newcommand{\radialsample}{\mathbf{b}}

\title{
  \textbf{DiffNash: Finding Nash Equilibrium in High-Dimensional
  Continuous Action Spaces with Diffusion Policies}
}

\author{
  Haichuan Wang$^{1,\dagger}$ \quad
  Anagha Satish$^{1,\dagger}$ \quad
  Jiale Zheng$^{1}$ \\
  Saisneha Ghatti$^{2}$ \quad
  Lingkai Kong$^{1}$ \quad
  Kai Wang$^{3}$ \quad
  Milind Tambe$^{1}$ \\
  \\
  $^{1}$Harvard University \\
  $^{2}$Massachusetts Institute of Technology \\
  $^{3}$Georgia Institute of Technology \\
  \\
  \texttt{haichuan\_wang@g.harvard.edu}
}
\date{}

\begin{document}

\maketitle

\begingroup
\renewcommand{\thefootnote}{$\dagger$}
\footnotetext{Equal Contribution.}
\endgroup

% --------------------------------------------------
% Main Paper
% --------------------------------------------------
\begin{abstract}

Computing equilibria in high-dimensional continuous games requires learning and representing rich, multi-modal strategy distributions, something conventional neural policies struggle to achieve. Diffusion models offer a natural solution, but standard diffusion training requires direct samples from a fixed target distribution, which are generally unavailable under iterative game-solving updates. We propose a diffusion-based framework for equilibrium computation that enables diffusion-based variants of Online Mirror Descent (OMD) and Fictitious Play (FP). Our approach derives a reweighted score-matching objective that trains diffusion policies without requiring direct samples from each updated policy. Our main method, $\textbf{DiffNash}$, combines the fast convergence of OMD with the sample efficiency of FP by reusing past actions through a replay buffer. To correct the resulting distribution shift from sample reuse, we derive an unbiased estimator of the difference between  current and past policies’ Evidence Lower Bounds, whose exponential provides a tractable proxy for the corresponding importance weight without evaluating diffusion likelihoods. We evaluate our approach on forest protection games and continuous Colonel Blotto games, reducing the Nash gap over the strongest baseline by $\textbf{35.6\%}$ and $\textbf{24.7\%}$, respectively, with minimal hyperparameter tuning.

\end{abstract}
\section{Introduction}

In real-world decision making, many choices are strategic: a party's outcome depends not only on its own decisions, but also on those of adversaries with competing objectives. Game theory provides a framework for modeling these interactions. In many applications, however, strategic decisions are not naturally enumerable and instead lie in continuous action spaces \citep{stein2008separable, ganzfried2021algorithm,kroer2015discretization}. In spatial security games, for example, a defender allocates resources while an adversary chooses where to attack \citep{kamra2018policy,gan2017security,alpern2022continuous}; these decisions are naturally represented by locations or trajectories rather than discrete actions. Continuous decisions also arise in economic competition, where firms choose production quantities \citep{shi2020multi}, and in resource-allocation games, where players divide continuous budgets across competing objectives \citep{roberson2006colonel}. Such settings are difficult to capture with finite strategy spaces and require reasoning over distributions on large continuous action spaces.

A Nash equilibrium is a strategy profile in which neither player can improve its outcome by unilaterally changing its strategy. A mixed strategy is a probability distribution over a player's possible actions; in continuous games, this distribution may be defined over an infinite action space \citep{glicksberg1952further}. As the dimensionality and complexity of the action space grow, representing the mixed strategy itself becomes challenging, since equilibrium play may require complex, multimodal distributions over high-dimensional continuous actions \citep{perkins2015mixed}.

Continuous-action games, even in the more well-understood two player zero-sum (2p0s) setting, present unique computation challenges. While efficient algorithms \citep{dantzig1951proof} and practical solvers \citep{jain2011double,cerny2024layered} exists for 2p0s games with finite action set, they cannot directly handle infinitely many actions in continuous-action games. Therefore, prior works \citep{kamra2019deep,ijcai2023p317} parameterize player strategies with neural networks and optimizes them through iterative algorithms such as Fictitious Play (FP) or Online Mirror Descent (OMD) to approximate Nash equilibria. However, these neural policies can suffer from mode collapse in high-dimensional action spaces, failing to capture the multimodal distributions that naturally arise in continuous-action games. Since diffusion model is capable of modeling multimodality in continuous vectors \citep{ho2020denoising,wang2023diffusion,li2024learning}, we ask 

\begin{center}
\emph{%To better serve users,
Can we parameterize player strategy with diffusion policies in iterative game-solving algorithms?}
\end{center}

In this work we propose \textbf{DiffNash}, a method for introducing diffusion models into equilibrium computation for games with continuous action. We parameterize each player's mixed strategy with a diffusion model and use Online Mirror Descent (OMD) \cite{blair1985problem} to iteratively update these strategies towards equilibrium. At each iteration, we sample candidate actions from the player's current diffusion policy and evaluate how useful those actions would be against the opponent's current policy. Standard diffusion training assumes access to samples from the target distribution, while the OMD update defines this distribution only implicitly. Hence, we adapt the reweighting technique of \citet{ma2025efficient} to modify the diffusion training loss, provably recovering the correct target update. To reduce the cost of repeatedly sampling news actions from diffusion and stabilize training along OMD iterations, we reuse samples from past policies through a replay buffer. We correct the resulting distribution shift with importance weighting and approximate the importance weights using Evidence Lower Bounds (ELBOs).

Our main contributions are: (1) We develop a diffusion-policy framework for equilibrium computation in continuous action games using utility-weighted score matching to represent complex, potentially multimodal mixed strategies (\Cref{sec:rsm}) (2) We incorporate a diffusion update into OMD, with an empirical comparison showing substantially faster convergence than a corresponding diffusion-FP approach (\Cref{sec:comparison}) (3) We construct an ELBO-difference based importance weighting that allows past actions to be reused for opponent evaluation (\Cref{subsec:efficient-IS}) and stabilizing training (\Cref{subsec:stailizer}) (4) We evaluate our approach on two continuous-action games, reducing the Nash gap over the strongest baseline by $\textbf{35.6\%}$ and $\textbf{24.7\%}$, respectively, with minimal hyperparameter tuning (\Cref{sec:experiments}).
\section{Background}
% \textbf{Notation} We use superscripts for game-theoretic notation, such as player and equilibrium-solver iteration indices, and subscripts for diffusion model notation. For a vector $\vaction \in \mathbb{R}^d$, we denote its $j$-th coordinate by $[\vaction]_j$. We use a parenthesized superscript $a^{(n)}$ to denote the $n$-th sample from a distribution $\mathcal D$. More details are discussed below. 

\subsection{Problem Formulation}
We consider a two-player normal-form game $\mathcal{G}=(\mathcal{A}^1,\mathcal{A}^2,u^1,u^2)$, where $\mathcal{A}^i \subseteq \mathbb{R}^{d^i}$ is a non-empty compact action space of
player $i$ and
$u^i:\mathcal{A}^1\times\mathcal{A}^2\rightarrow\mathbb{R}$
be its bounded and measurable utility function \footnote{We use superscripts for game-theoretic notation, such as player and equilibrium-solver iteration indices, and subscripts for diffusion model notation.}. We consider high-dimensional continuous action set, which may also subject to global feasibility constraints, such as a fixed resource
budget or a bounded spatial domain. A mixed strategy $\ownpolicy\in\Delta(\mathcal{A}^i)$ is a probability
distribution over $\mathcal A^i$. Given a strategy profile
$\pi=(\pi^1,\pi^2)$, the expected payoff to player 1 is
$U^i(\pi^1,\pi^2)
=
\mathbb{E}_{\substack{
    \vaction^1\sim\pi^1\\
    \vaction^2\sim\pi^2
}}
\left[u^i(\vaction^1,\vaction^2)\right]
$. Our goal is to find the Nash equilibrium of the game, where neither player has the incentive to deviate from their chosen strategy. Formally, 
\begin{definition}[Nash Equilibrium]
$\pi^\star=(\pi^{1,\star},\pi^{2,\star})$
is a Nash equilibrium if, for $i\in\{1,2\}$, $U^i(\pi^{i,\star},\pi^{-i,\star})
\geq
U^i(\pi^i,\pi^{-i,\star})
\qquad
\text{for all }\pi^i\in \Delta(\mathcal{A}^i).$
\end{definition}

We focus primarily on two-player zero-sum (2p0s) games, where $u^1(\vaction^1,\vaction^2)=-u^2(\vaction^1,\vaction^2)$. 2p0s games with finite action can be solved by linear programming \citep{dantzig1951proof}. For games with large combinatorial action spaces, such as network security games, practical methods exploit problem structure and propose practical solvers \citep{jain2011double,cerny2024layered}. Continuous-action games require representing and optimizing distributions over an infinite action space. Recent approaches therefore use neural parameterizations of mixed strategies within iterative game-solving algorithms \citep{kamra2019deep, ijcai2023p317}. Our work follows this latter setting, using diffusion models as expressive representations of continuous mixed strategies.

We primarily study the 2p0s setting, where equilibrium objective and associated learning dynamics are well understood. Methods with strong convergence guarantees in 2p0s games are also commonly studied and applied in general-sum settings, although convergence to Nash equilibrium is no longer guaranteed~\citep{anagnostides2022optimistic}. We also include general-sum games, where player's objectives are not necessarily opposed and equilibrium computation is generally more challenging. We measure solution quality using the \emph{Nash gap}, defined as the total improvement available to players through unilateral best responses \citep{lanctot2017unified}. It is non-negative and equals zero at Nash equilibrium, measuring how exploitable the strategy profile is.

\begin{definition}[Nash Gap]\label{def:nash-gap}
     $ \operatorname{Gap}(\pi_1,\pi_2)
=
\max_{\pi_1'\in\Delta(\mathcal{A}_1)}
    U(\pi_1',\pi_2)
-
\min_{\pi_2'\in\Delta(\mathcal{A}_2)}
    U(\pi_1,\pi_2').$
\end{definition}

A low Nash gap limits gains from opponent's unilateral deviations. We seek such strategies in high-dimensional action spaces, where conventional neural policies struggle \citep{kamra2019deep}.

\subsection{Diffusion Models}
\label{sec:diffusion-background}
Denoising diffusion probabilistic models (DDPMs) \citep{sohl2015deep, ho2020denoising} is a generative model consisting of two stochastic processes: one forward process which progressively adds noise to the data, and a reverse (denoising) process which learns to gradually remove noise. Formally, given a clean data from the target distribution $\vsample_0 \sim p_0$, the forward process gradually perturbs the data by $q(\vsample_t \mid \vsample_{t-1})
= \mathcal{N}\!\left(\vsample_t \mid \sqrt{1-\beta_t}\vsample_{t-1},\, \beta_t \mathbf{I}\right)$, where $\vsample_t$ is the perturbed data at time $t$ and $\beta_t$ is the noise variance at time step $t = 1, \cdots, T$. The corresponding joint distribution of forward trajectory is $q_{0:T}(\vsample_{0:T})=p_0(\vsample_0)\prod_{t=1}^T q_{t|t-1}(\vsample_t|\vsample_{t-1})$. Since each step adds independent Gaussian noise, we can sample $\vsample_t$ by $\vsample_t\sim p_t(\vsample_t)
= \int p_0(\vsample_0)\,q_{t\mid 0}(\vsample_t\mid \vsample_0)\,\mathrm{d}\vsample_0$, where $q_{t\mid 0}(\vsample_t\mid \vsample_0)
:= \mathcal{N}\!\left(
    \vsample_t;\sqrt{\bar{\alpha}_t}\,\vsample_0,\,
    (1-\bar{\alpha}_t)\mathbf{I}
\right),
\
\bar{\alpha}_t:=\prod_{\ell=1}^{t}\alpha_\ell,
\ 
\alpha_t:=1-\beta_t$. \footnote{Both $p$ and $q$ are density function. We use $q$ to denote tractable distributions like Gaussian corruption. When there is no ambiguity, we use $p$ and $p_0$ interchangably.} 

The backward process recovers the data distribution from a noise distribution $p_T$ with a iterative process. When $\beta_t$ is small, one can approximate the denoising process by the score function $s(\vsample_t) := \nabla_{\vsample_t} \log p(\vsample_t)$. Specifically, $q(\vsample_{t-1}|\vsample_t) \approx N(\vsample_{t-1}; \frac{1}{\alpha_t}(x_t + \beta_t s(x_t)), \sigma_t^2 \mathbf{I})$, where $\sigma_t = \frac{1-\bar{\alpha}_{t-1}}{1-\bar{\alpha_t}}\beta_t$. Therefore, one can follow the following iterative procedure to draw samples from $p_0$:
\begin{align}\label{eq:diff-sampling}
    \vsample_{t-1} = \frac{1}{\alpha_t} (\vsample_t + \beta_t s(\vsample_t)) + \frac{1- \bar{\alpha}_{t-1}}{1-\bar{\alpha}_t}\beta_t \epsilon_t
\end{align}
for $t$ from $T$ to $1$ and $\epsilon_t \sim \mathcal N(0,1)$. However, we do not know $s(\vsample_t)$ in closed form, so we train a score network $s_\theta(\vsample_t, t)$ to approximate it, with the following tractable loss function \citep{ho2020denoising}:
\begin{align}\label{eq:ddpm-score-loss}
   L(\theta)= \mathbb E_{t \in [T]}(1-\bar{\alpha}_t) \mathbb{E}_{\vsample_0 \sim p_0, \vsample_t \sim q_{t|0}}[||s_\theta(\vsample_t;t) - \nabla_{\vsample_t}\log q_{t|0}(\vsample_t|\vsample_0)||^2]
\end{align}
\citet{ma2025efficient} shows that fix a time $t$, \Cref{eq:ddpm-equivalent-loss} has the same global minimizer as \Cref{eq:ddpm-score-loss}:
\begin{align}\label{eq:ddpm-equivalent-loss}
    L(\theta;t) = \mathbb E_{\vsample_t \sim p_t}[||s_\theta(\vsample_t;t) - \nabla_{\vsample_t} \log p_t(\vsample_t)]
\end{align}
Though \Cref{eq:ddpm-equivalent-loss} is intractable because we do not know $\nabla_{\vsample_t} \log p_t(\vsample_t)$, it's important for the reweighting technique used in \Cref{sec:rsm}.

% \newpage 

% With large enough $T$, the repeated noise gradually transforms the data distribution into (approximately) Gaussian distribution, i.e. $q(\vsample_t) \approx N(0, T\bet\vaction^2I)$. 

% \newpage 

% \haichuan{Will finish rewrite this section by tomorrow morning, also need to rewrite Proposition 2 theorem statement and all notations in this paper}

% Denoising diffusion probabilistic models (DDPMs) learn complex distributions by
% reversing a gradual Gaussian noising process~\citep{ho2020denoising}. Given a
% clean sample $a_0 \sim p_{\mathrm{data}}$, the forward process is
% \[
%     q(a_t \mid a_0)
%     =
%     \mathcal{N}\!\left(
%         \sqrt{\bar{\alpha}_t}\,a_0,\,
%         (1-\bar{\alpha}_t)I
%     \right),
% \]
% where $\alpha_t=1-\beta_t$ and
% $\bar{\alpha}_t=\prod_{s=1}^{t}\alpha_s$. Equivalently,
% \[
%     a_t
%     =
%     \sqrt{\bar{\alpha}_t}\,a_0
%     +
%     \sqrt{1-\bar{\alpha}_t}\,\epsilon,
%     \qquad
%     \epsilon\sim\mathcal{N}(0,I).
% \]

% A neural network $\epsilon_\theta(a_t,t)$ is trained to predict the injected
% noise using the denoising objective
% \[
%     \mathcal{L}_{\mathrm{diff}}(\theta)
%     =
%     \mathbb{E}_{a_0,t,\epsilon}
%     \left[
%         \left\|
%             \epsilon-
%             \epsilon_\theta(a_t,t)
%         \right\|_2^2
%     \right].
% \]
% At inference time, samples are generated by starting from Gaussian noise and
% iteratively applying the learned reverse denoising process. In our setting,
% $a_0$ represents a player's action, so the resulting diffusion model
% parameterizes the player's mixed strategy $\pi_\theta$.

\subsection{Equilibrium Solvers: Online Mirror Descent and Fictitious Play}
\label{sec:equilibrium-solvers}

% \[
%     u_i(a_i;\opponentpolicy)
%     :=
%     \mathbb{E}_{a_{-i}\sim\opponentpolicy}
%     [u_i(a_i,a_{-i})]
% \]

Let $u^i(\ownaction; \opponentpolicy) :=
    \mathbb{E}_{\opponentaction \sim \opponentpolicy}
    [u^i(\ownaction,\opponentaction)]$
denote player $i$'s expected utility for action $\ownaction$ against opponent
strategy $\opponentpolicy$. We introduce two major iterative algorithms used to compute Nash equilibrium for continuous action in ML literature: Online Mirror Descent (OMD) and Fictitious Play (FP). 

\textbf{Online mirror descent.}
In OMD \citep{kivinen1997exponentiated,beck2003mirror}, each player computes a KL-regularized best response to the opponent’s strategy from the previous iteration. At iteration $k$, player i computes
\begin{align}\label{eq:mirror-decent-update}
     \pi^{i,k+1}(\ownaction) = \pi^{i,k}(\ownaction)\exp(\frac{1}{\eta} u_i(\ownaction;\pi^{-i,k}))/Z_{\mathbf{OMD}}^{i,k}\ ,
\end{align}
where $\frac{1}{\eta}$ is called temperature and $Z_{\mathbf{OMD}}^{i,k}
=
\int_{\mathcal A^i}
\pi^{i,k}(\ownaction)
\exp\!\left(
\frac{1}{\eta}\,u_i(\ownaction;\pi^{-i,k})
\right)
\, da^i$ is an intractable unknown constant. We define empirical average of OMD responses as $\Pi^{i,k} = \frac{1}{k}\sum_{h=1}^{k} \pi_{\mathbf{OMD}}^{i,h}$. Under standard regularity conditions and an appropriate learning rate schedule, OMD is a no-regret algorithm on continuous action space \citep{krichene2015hedge}. When both players play such algorithms, their empirical strategies converge asymptotically to the Nash equilibrium set \citep{balandat2016minimizing}. 

\textbf{Fictitious play.}
While in OMD player computes a KL-regularized best response to the opponent’s current strategy, in FP \citep{brown1951iterative} player  best responds to the empirical distribution of the opponent's past play \citep{brown1951iterative}. FP is known to converge to Nash equilibrium with continuous action space and continuous utility function \citep{danskin1954fictitious}. We consider a special variant called Generalised Weakened Fictitious Play (GWFP) \citep{leslie2006generalised}. Let \(\pi_{\mathbf{GWFP}}^{i,k}\) denote player \(i\)'s best response at iteration \(k\), and let $\Pi^{i,k} := \frac{1}{k}\sum_{h=1}^{k} \pi_{\mathbf{GWFP}}^{i,h}$ denote the empirical average of its past responses. Under GWFP, player $i$ computes
\begin{align}\label{eq:FP-update-rule}
\pi_{\mathbf{GWFP}}^{i,k+1}
=
\exp\left(\frac{1}{\eta^k}
u^i(\ownaction,\Pi^{-i,k})\right)/{Z^{i,k}}  \ ,
\end{align}

% B_{k+1}^i
% \in
% \arg\max_{\ownpolicy \in \Delta(\mathcal A_i)}
% u_i\left(\ownpolicy,\Pi_k^{-i}\right).
where $\eta^k$ is called temperature and $Z^{i,k} =
\int_{\mathcal A^i}
\exp\left(
\frac{u^i(a^i,\Pi^{-i,k})}{\eta_k}
\right)\mathrm{d} \vaction^i$ is an intractable normalization constant. Under vanishing temperature condition, $(\Pi^{1,k}, \Pi^{2,k})$ converges to the set of Nash equilibria asymptotically under GWFP \citep{leslie2006generalised}.

To represent multimodal mixed strategies, we represent player $i$’s latent action with a diffusion model $p_{\theta^i}$. A game-specific map $f^\mathcal G$ will convert the latent sample to a feasible player action, i.e., $\vaction^i=f^{\mathcal G}(\vsample), \vsample\sim p_{\theta^i}$. The resulting mixed strategy is then $\pi_{\theta^i} = f_\#^\mathcal G p_{\theta^i}$. We call $\pi_{\theta^i}$ player $i$’s diffusion policy following \citet{ma2025efficient}'s notation, abstracting away latent space details.

\textbf{Understanding Convergence Dynamics.}
Computing the Nash gap requires solving an intractable best response problem, and approximating it at every iteration is prohibitively expensive. To study convergence dynamics, we use a multi-market Cournot game \citep{kamra2019deep} whose unique Nash equilibrium is an analytically derivable singleton, which we denote as $(\vaction^{1,*}, \vaction^{2,*})$. In Cournot game, two firms compete by choosing production level across $d$ independent markets to maximizes total profits. We defer Cournot game's details to Appendix \ref{app:cournot}. We define the distance of a strategy profile $\Pi$ to the equilibrium by $\frac{1}{2} \sum_{i=1}^{2} \mathbb{E}_{a^i \sim \pi^{i}} ||a^{i}-a^{i, \star}||_2.$ Because this distance is easy to calculate, we compute it at each iteration in the Cournot game to study convergence dynamics.

\section{Proposed Method}

% \Cref{sec:rsm} derives reweighted score-matching objectives for diffusion-based GWFP and OMD. In \Cref{sec:comparison}, we compare both methods and select OMD as our core backbone due to its superior convergence speed. To reduce repeated diffusion sampling and stabilizes training, \Cref{subsec:efficient-IS} and \Cref{subsec:stailizer} introduce importance sampling to reuse past iteration samples, completing \textsc{DiffNash}.

% We first establish reweighted score-matching objectives ($\mathcal{L}^{\mathrm{GWFP}}$ and $\mathcal{L}^{\mathrm{OMD}}$) in Section~4.1. In Section~4.2, we compare both methods and select OMD as our core backbone due to its superior convergence speed. To eliminate OMD's high sampling overhead, Section~4.3 introduces ELBO-based importance sampling to reuse past opponent actions, while Section~4.4 extends this scheme to self-actions to form our complete algorithm, \textsc{DiffNash}.

\begin{figure*}[t]
    \centering
\includegraphics[width=0.9\textwidth]{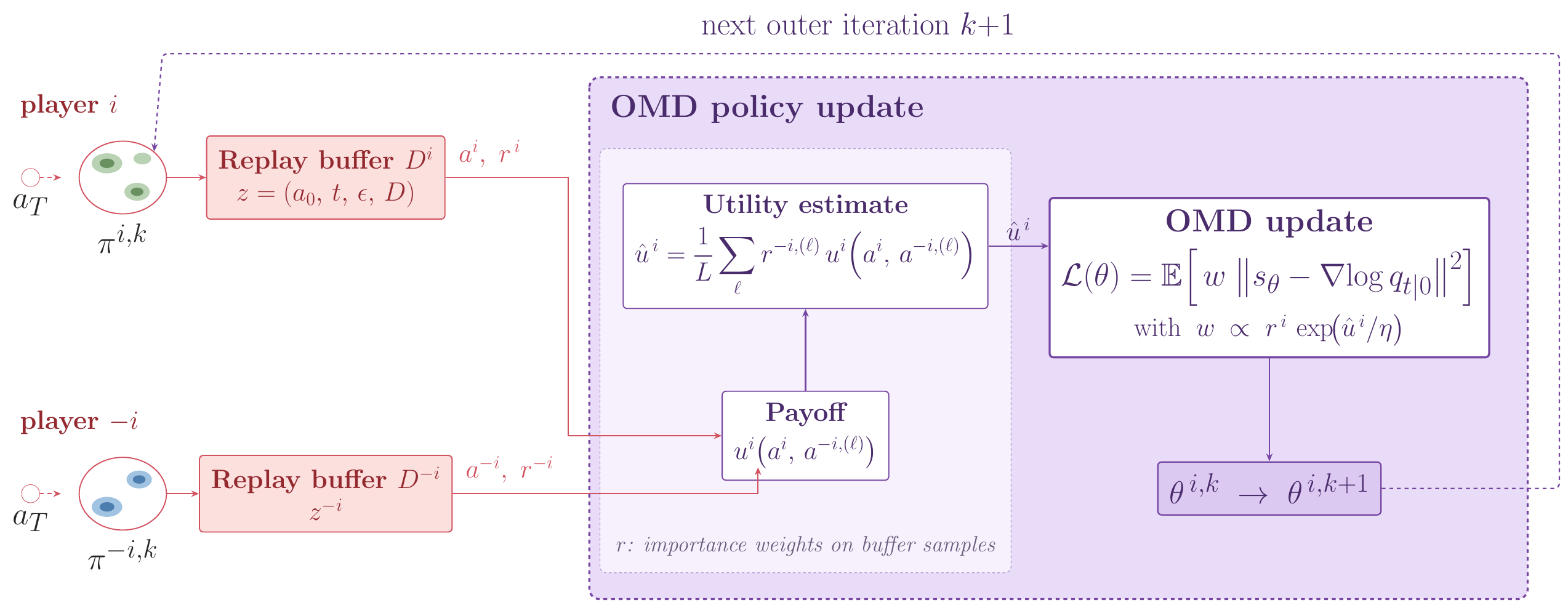}
    \caption{\textbf{Overview of DiffNash.} Each player samples actions from a diffusion policy into a replay buffer. Both lanes pass through a shared replay weight, obtained from a single-call ELBO difference (Proposition \ref{prop:elbo-difference}), and scored against the opponent's buffered actions by importance sampling. The OMD weight then reweights the denoising loss. Player $-i$ is updated symmetrically.}
    \label{fig:main}
\end{figure*}

\subsection{Enabling Diffusion-based GWFP and OMD via Reweighted Score Matching}
\label{sec:rsm}

% The equilibrium updates in Section~\ref{sec:equilibrium-solvers} define target action distributions whose closed-form solutions in \Cref{eq:FP-update-rule,eq:mirror-decent-update} contain unknown normalization constants. Such probablistic models with unknown normalization constants is called energy-based model, which is  difficult to train and sample from \citep{song2021train}. 

% whereas standard diffusion training requires samples from the target distribution. We therefore use reweighted score matching \citep{ma2025efficient} to learn the score functions of the target distributions $\pi^{\mathbf{MaxEnt}}$ and $\pi^{\mathbf{MD}}$ for FP and OMD, respectively, without directly sampling from them.
% equilibrium updates in Section~\ref{sec:equilibrium-solvers} define a
% target distribution over actions. The closed-form solution of both \Cref{eq:FP-update-rule} and \Cref{eq:mirror-decent-update} have unknown normalization constants. 

The game-solving updates in GWFP \Cref{eq:FP-update-rule} and OMD \Cref{eq:mirror-decent-update} both involve unknown normalization constants. Such unnormalized distributions can be viewed as energy-based models, which are generally difficult to train and sample from \citep{song2021train}. Moreover, standard diffusion training requires samples from the target distribution, which are not directly available for the GWFP and OMD updates. \citet{ma2025efficient} show that, instead of integrating the squared score loss in \Cref{eq:ddpm-equivalent-loss} with respect to $p_t$, the objective can be reweighted by any strictly positive function $g(\vaction_t)$ without changing its global minimizer:
\begin{align}\label{eq:reweighted-formula}
    L(\theta;t) = \int g(\vaction_t)||s_\theta(\vaction_t;t)-\nabla_{\vaction_t}\log p_t(\vaction_t)||^2d\vaction_t
\end{align}
By choosing appropriate $g(\vaction_t)$, they derive tractable loss to learn the score function of diffusion policy in online RL. We use $\pi_{t,\solver}^{i,k+1}(\vaction_t)$ to denote noise-corrupted target update at time $t$, i.e., $\pi_{t,\solver}^{i,k+1}(\vaction_t)
:=
\int \pi_{\solver}^{i,k+1}(\vaction_0)\,
q_{t\mid 0}(\vaction_t\mid \vaction_0)\,
\mathrm{d}\vaction_0$, where $\solver \in \{\GWFP, \OMD\}$. We adapt the reweighting technique to our setting by setting $p_t=\pi_{t,\solver}^{i,k+1}$ and choosing proper $g(\vaction_t)$, so that the resulting objectives can recover score functions of noise-corrupted updated in OMD and GWFP.

% address this, we apply a reweighted score matching loss, a broader loss family that provably learn the score function corresponding to $\pi_{\mathbf{OMD}}^{i, k+1}(\ownaction)$ and $\pi_{\mathbf{GWFP}}^{i, k+1}(\ownaction)$ \citep{ma2025efficient}.

\begin{restatable}[OMD and GWFP Loss]{proposition}{gwfpomdloss}
\label{lem:new-loss}
By choosing an appropriate $g_{\OMD}(\vaction_t)$ and substitute it into \Cref{eq:reweighted-formula}, we can equivalently write their loss functions as
\begin{equation}\label{eq:OMD-loss}
\mathcal{L}_{\mathrm{OMD}}(\theta;t)
=
\mathbb{E}_{\vaction_0\sim\pi^{i,k}}
\left[
\mathbb{E}_{\vaction_t\sim q_{t|0}(\cdot\mid \vaction_0)}
\left[
\exp\left(\frac{1}{\eta}u^i(\vaction_0;\pi^{-i,k})\right)
\left\|
s_{\theta}(\vaction_t,t)
-
\nabla_{\vaction_t}\log q_{t|0}(\vaction_t\mid \vaction_0)
\right\|_2^2
\right]
\right].
\end{equation}
For every $t \in (0, T]$, the minimizer of \Cref{eq:OMD-loss} over $s$ is the score function of 
$\pi^{i,k+1}_{t, \OMD}(\vaction_t)$. In Appendix \ref{app:proof-rsm}, we show with a proper $g_{\GWFP}(\vaction_t)$, we can recover the score function of $\pi^{i,k+1}_{t, \GWFP}(\vaction_t)$.
% Similarly, substituting a proper chosen $g_{\GWFP}(\vaction_t)$ into \Cref{eq:reweighted-formula}, we derive a $\mathcal{L}_{\mathrm{GWFP}}(\theta;t)$ that recovers the score function the score function of 
% $\pi^{i,k+1}_{t, \GWFP}(\vaction_t)$ at the global minimum. 
\end{restatable}

We defer \Cref{lem:new-loss}'s proof to Appendix \ref{app:proof-rsm}. The intuition of choosing $g({\vaction_t})$ is to ensure we can sample from all distributions appear in the resulting loss. For example, at iteration \(k+1\), we have sampling access to \(\pi^{i,k}\) and $q_{t|0}$, while \(u^i(\tilde{\vaction}_0;\pi^{-i, k})\) can be unbiasedly estimated with samples from opponent's current strategy $\pi^{-i, k}$. Thus, the objectives in \Cref{lem:new-loss} are tractable, and optimizing them recover the corresponding score function \(\nabla_{\vaction_t}\log \pi_{t,\solver}^{i,k+1}(\vaction_t)\). Recall we only need the score function to perform the iterative diffusion sampling process specified in \Cref{eq:diff-sampling}. Hence, these objectives suffice to implement diffusion-based OMD and GWFP updates.

\subsection{Comparison Between FP and OMD}\label{sec:comparison}
\begin{wrapfigure}{r}{0.45\textwidth}
    \centering
    \includegraphics[width=\linewidth]{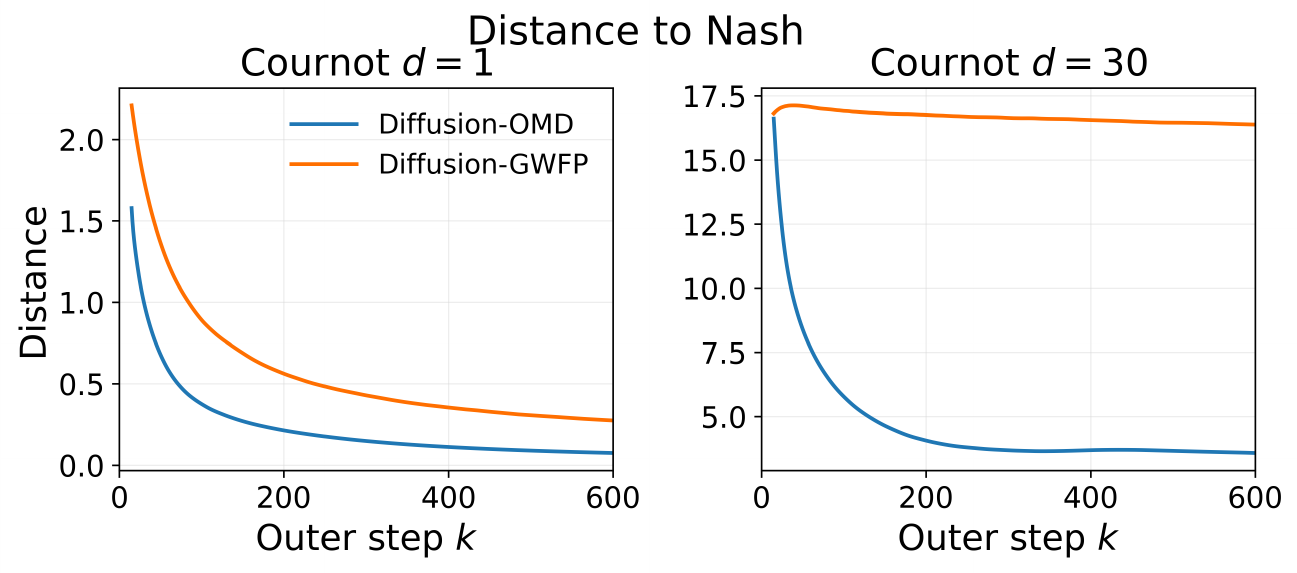}
    \caption{We run diffusion-based GWFP and OMD on the Cournot game for $600$ iterations across three random seeds and report their mean distances to the Nash equilibrium.}
    \label{fig:cournot_gwfp_vs_omd}
\end{wrapfigure}

In this section, we examine the complementary strengths of FP and OMD to guide the design of an algorithm that combines their advantages. \citep{perolat2021scaling,lauriere2022scalable} report faster convergence for OMD, partly because best responding to stale historical strategies can slow FP's convergence. We confirm this finding in \Cref{fig:cournot_gwfp_vs_omd}: we run diffusion-based GWFP and OMD for $600$ iterations on the Cournot game. OMD converges substantially faster in both dimensionalities. With action dimension $d=30$, GWFP's distance to equilibrium decreases only marginally, highlighting the difficulty of convergence in high-dimensional action spaces. While prior works are often based on FP \citet{kamra2019deep,karthikeyan2025difffp}, this result motivates using OMD as our equilibrium solver's backbone.

Evaluating \(\mathcal{L}_{\OMD}(\theta;t)\) in
\Cref{eq:OMD-loss} requires estimating the expected utility against the
current opponent policy. Using \(L\) freshly generated opponent actions,
this expectation can be estimated as $\widehat{u}^i(\vaction_0;\pi^{-i,k})
=
\frac{1}{L}
\sum_{\ell=1}^{L}
u^i\!\left(\vaction_0,\opponentsample \right),\opponentsample \sim \pi^{-i,k}.
$
However, since diffusion generation involves an
iterative denoising process, repeatedly sampling fresh actions from the
current opponent policy can be computationally expensive. FP largely avoids this cost because each player best-responds to the
opponent's historical mixture \(\Pi^{-i,k}\). Following
\citet{kamra2019deep}, samples from this mixture can be maintained in a
replay buffer, making them cheap to reuse. To improve
the sample efficiency of our OMD solver, we similarly maintain a replay buffer for
each player and use importance sampling (IS) to correct the distribution shift between the policy that generates a sample and $\pi^{-i,k}$. Formally, let $k^{(\ell)}$ denote the iteration
at which the \(\ell\)-th sampled action was generated, so that
$\vaction_{(\ell)}^{-i,k^{(\ell)}} \sim \pi^{-i, k^{(\ell)}}$. We then have
\begin{equation}\label{eq:opponent-replay-is}
\widehat{u}^i(\vaction_0;\pi^{-i,k})
=
\frac{1}{L}
\sum_{\ell=1}^{L}
r^{-i,k,k^{(\ell)}}
u^i\!\left(\vaction_0,\vaction_{(\ell)}^{-i,k^{(\ell)}}\right),
\ \text{where } r^{-i, k, k^{(\ell)}} = \frac{
    \pi_{\theta^{-i,k}}\!\left(\vaction_{(\ell)}^{-i,k^{(\ell)}}\right)
}{
    \pi_{\theta^{-i,k_\ell}}\!\left(\vaction_{(\ell)}^{-i,k^{(\ell)}}\right),
}
\end{equation}
$r^{-i, k, k^{(\ell)}}$ is sample $l$'s importance weight\footnote{In practice, we use $\min\left\{
r^{-i,k,k^{(\ell)}},
C
\right\}$ as the weight in \Cref{eq:opponent-replay-is}, where $C$ caps the weight to prevent excessively large values \citep{schulman2017proximal}.}. However, computing \(\pi_\theta(\cdot)\) in $r^{-i, k, k^{(\ell)}}$ exactly remains expensive for diffusion models \citep{song2020score}, motivating us to develop a principled proxy.

\subsection{Enable Importance Sampling}\label{subsec:efficient-IS}
% Fix a player $i$. Suppose $\vaction_0$ was sampled from the behavior policy $\pi_i^b$ at iteration b. To reuse this sample when estimating an expectation under the policy $\pi_i^k$ at iteration k, we weight its contribution by the importance ratio $\frac{\pi_i^k(\vaction_0)}{\pi_i^b(\vaction_0)}$. However, computing the exact log likelihood $\log \pi_\theta(\vaction_0)$ is generally expensive for diffusion models.

In this section, we fix a player and suppress the player index. Consider a diffusion sample $\vaction_0$. 
\begin{lemma}\label{lemma:log-is}
    $r^{k,k^{(\ell)}} =
    \exp\left(
        \log \pi_{\theta^{k}}\!\left(\vaction_0)\right)
        - \log \pi_{\theta^{k^{(\ell)}}}\!\left(\vaction_0\right)
    \right)$.
\end{lemma}
\Cref{lemma:log-is} is proved by direct calculation. By \Cref{lemma:log-is}, we only need to approximate $\log \pi_\theta(\vaction_0)$ to enable efficient importance sampling. We achieve this by using evidence lower bounds (ELBO). 
\begin{definition}
   The ELBO of a sample $\vaction_0$ is  $\ELBO_{\theta}(\vaction_0)
\coloneqq
\mathbb{E}_{q(\vaction_{1:T}\mid \vaction_0)}
\left[
\log
\frac{p_{\theta}(\vaction_{0:T})}
     {q(\vaction_{1:T}\mid \vaction_0)}
\right]$.
\end{definition}

ELBO lower bounds the log-likelihood, i.e., $\ELBO_\theta(\vaction_0) \leq \log \pi_\theta(\vaction_0)$ \citep{ho2020denoising}. Since $\log \pi_\theta(\vaction_0)$ is intractable for diffusion, the ELBO is commonly used as a tractable surrogate for training diffusion models. \citet{mcallister2026flow} also uses ELBO to approximate IS weight, though their method is for flow matching policy. We develop new method for diffusion policy. Substituting log likelihood with ELBO in \Cref{lemma:log-is}, it remains to compute $\ELBO_{\theta^{k}}(\vaction_0)-\operatorname{ELBO}_{\theta^{k^{(\ell)}}}(\vaction_0)$ to obtain IS weights.  Although evaluating each ELBO is still expensive \citep{ho2020denoising}, \Cref{prop:elbo-difference} shows that their difference admits an unbiased estimator requiring only one score network evaluation.
% To approximate $r^{-i,k, k^{(\ell)}}$, it therefore remains to compute $\ELBO^{-i,k}(\vaction_0)-\operatorname{ELBO}^{-i,k^{(\ell)}}(\vaction_0)$ and exponentiate it. Although evaluating each ELBO in full remains expensive, \Cref{prop:elbo-difference} shows that their difference admits an unbiased estimator requiring only one score network evaluation.
\begin{restatable}{proposition}{noiseprediction}\label{prop:elbo-difference}
    Suppose $\pi^{\behavioriteration}$ and $\pi^{k}$ share the same forward-noising
schedule $\{\beta_t\}_{t=1}^T$. Denote $D(\theta; \vaction_0, t, \epsilon) = C_t||s_{\theta}(\vaction_t, t)+\frac{\epsilon}{\sqrt{1-\bar{\alpha}_t}}||_2^2$, where $C_t$ a constant and $\vaction_t = \sqrt{\bar{\alpha}_t}\vaction_0 + \sqrt{1-\bar{\alpha_t}}\epsilon$. If $t \sim\operatorname{Unif}\{1,\ldots,T\}$ and $\epsilon \sim \mathcal N(0,I)$, then $
\isestimator
:=
T\bigl(D(\theta^{\behavioriteration}; \vaction_0, t, \epsilon)-D(\theta^{k}; \vaction_0, t,\epsilon)\bigr)
$ is ELBO difference's unbiased estimator:
$\mathbb{E}_{t,\epsilon}
[\isestimator]
=
\operatorname{ELBO}^{k}(\vaction_0)
-
\operatorname{ELBO}^{\behavioriteration}(\vaction_0).$
\end{restatable}
% \begin{proof}
%   The proof is deferred to \Cref{appdx:noise-prediction-form}.
% \end{proof}

 The proof is deferred to \Cref{appdx:noise-prediction-form}. In iteration $\behavioriteration$, when we sample $\vaction_{0} \sim \pi^{\behavioriteration}$, we record $\left[\vaction_{0} , t, \epsilon,D(\theta^{\behavioriteration}; \vaction_0 , t, \epsilon)\right]$ in a buffer, so that we only need to compute $D(\theta^{k}; \vaction_{0} , t, \epsilon)$ in future iteration $k$ to construct $\isestimator$. Drawing sample from $\pi_{\theta^{k}}$ requires $T$ iterative score network calls by \Cref{eq:diff-sampling}, while computing $D(\theta^{k}; \vaction_0 , t, \epsilon)$ only requires one. Hence, our method reduces much sampling cost. We plug $\exp(\widehat{\Delta}^{k, k^{(\ell)}})$ back to \Cref{eq:opponent-replay-is} to enable reusing opponent buffers $\Pi^{-i, k}$.

\subsection{Self-Action Importance Sampling Stabilizes Training}\label{subsec:stailizer}
  \begin{wrapfigure}{r}{0.35\textwidth}
    \centering
    \includegraphics[width=\linewidth]{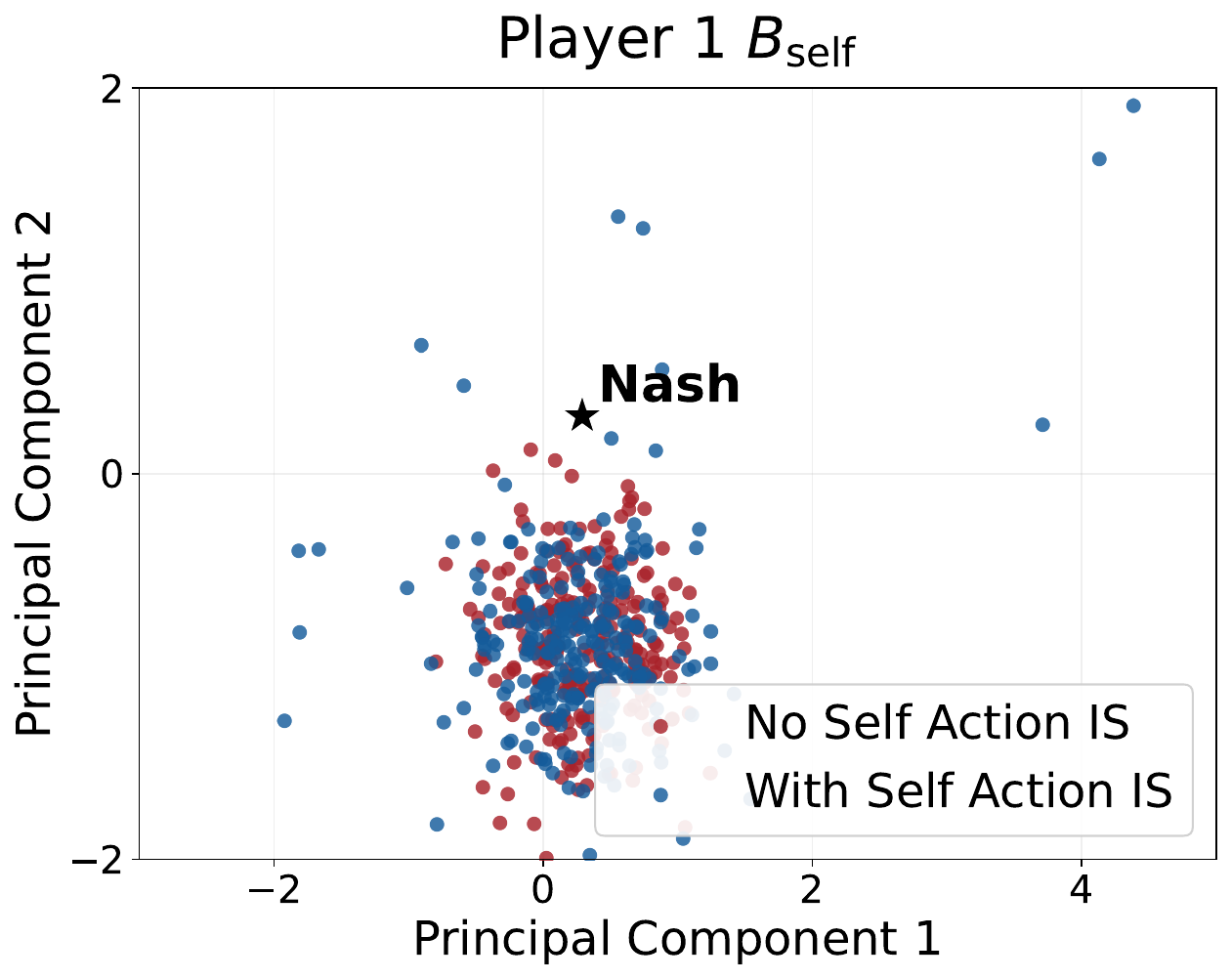}
    \caption{We use PCA to visualize the self-action buffer at iteration $101$ for Cournot $d=30$. }
    \label{fig:is_better}
\end{wrapfigure}
 
% Recall that implementing \Cref{eq:OMD-loss} requires
% $\vaction_0\sim\pi^{i,k}$. For a fixed $\vaction_0$, define the conditional loss
% $
% G_\theta^i(\vaction_0,t)
% :=
% \mathbb{E}_{\vaction_t\sim q_{t|0}(\cdot\mid \vaction_0)}
% \left[
% \exp\left(\frac{1}{\eta}u^i(\vaction_0;\pi^{-i,k})\right)
% \left\|
% s_{\theta_i}(\vaction_t,t)
% -
% \nabla_{\vaction_t}\log p_{t,\mathrm{OMD}}^{i,k+1}(\vaction_t)
% \right\|_2^2
% \right].
% $  Since we maintain a self-action replay buffer $\Pi^{i,k}$, we can reuse
% past actions through importance sampling instead of relying exclusively
% on fresh samples from the current policy. Let
% $B_{\mathrm{IS}}
% :=
% \left\lceil \zeta B_{\mathrm{self}}\right\rceil,
% \qquad
% B_{\mathrm{fresh}}
% :=
% B_{\mathrm{self}}-B_{\mathrm{IS}}.
% $
% For each replay sample $\ell\in\{1,\ldots,B_{\mathrm{IS}}\}$, let
% $b_\ell<k$ denote the iteration at which that sample was generated, so that
% $
% \widetilde \vaction_0^{(\ell)}\sim\pi^{i,b_\ell}.$ Hence, we have the following unbiased estimator for \Cref{eq:OMD-loss}.
% \begin{equation}
%     \widehat{\mathcal L}^{\mathrm{OMD}}(\theta;t)
% =
% \frac{1}{B_{\mathrm{self}}}
% \left[
% \sum_{\ell=1}^{B_{\mathrm{IS}}}
% r^{i,k, b^l}
% G_\theta^i\!\left(\widetilde \vaction_0^{(\ell)},t\right)
% +
% \sum_{j=1}^{B_{\mathrm{fresh}}}
% G_\theta^i\!\left(\vaction_0^{(j)},t\right)
% \right].
% \end{equation}

% \newpage  

Recall evaluating the OMD loss \Cref{eq:OMD-loss} requires samples from a player's self action $\pi^{i,k}$. Since we are maintaining replay buffers for both players, we can also use replay samples for a player's own action, not just opponent action. Denote a player's self action batch $\selfbatch$. Instead of always getting fresh samples from $\pi^{i,k}$, we use $\left\lceil \zeta |B_{\mathrm{self}}|\right\rceil$ IS samples with the rest from $\pi^{i,k}$. Similar to opponent buffer reuse, we use the IS weight to correct distribution shift induced by replay samples. We defer the full algorithm details to Appendix~\ref{app:full-algorithm}.

% $\selfis = \left\lceil \zeta B_{\mathrm{self}} \right\rceil$ IS samples and sample the rest $\selffresh = |\selfbatch|-\selfis$ from $\pi^{i,k}$. We use the following unbiased estimator  for \Cref{eq:OMD-loss}: $\frac{1}{\selfis}
% \sum_{l=1}^{\selfis}\isweight G_\theta^i(\vaction_{(\ell)}^{i,k^{(\ell)}}, t) + \frac{1}{\selffresh}\sum_{j=1}^{\selffresh}G_\theta^i(\vaction_{(j)}^{i,k}, t)$

% \begin{align*}
% \frac{1}{\selfis}
% \sum_{l=1}^{\selfis}\isweight G_\theta^i(\vaction_{(\ell)}^{i,k^{(\ell)}}, t) + \frac{1}{\selffresh}\sum_{j=1}^{\selffresh}G_\theta^i(\vaction_{(j)}^{i,k}, t),
% \end{align*}

We find that adding importance-weighted replay stabilizes training and prevents premature plateaus. The policy starts concentrate on a narrow set of high-utility actions as it approaches equilibrium, limiting the coverage of fresh samples. Replaying diverse past actions broadens the training distribution and prevents the policy from overfitting to a narrow region. We visualize the impact of replay mechanism in \Cref{fig:is_better}, where self action IS significantly improves coverage. We introduce an additional hyperparameter called trigger \(\kappa\), and we initialize the self-action IS once standard deviation of $\selfbatch$ falls below a fraction \(\kappa\) of its initial value. We analyze the impact of the self action IS mechanism in \Cref{sec:ablation}.

\section{Related Work}
\textbf{Equilibrium Computation.} We provide a comprehensive review in \Cref{appdx:additional-related-work} and focus here on learning-based approaches for continuous action spaces. Prior work learns generative approximate best responses \citep{kamra2019deep}, trains randomized policy networks with zeroth-order equilibrium dynamics \citep{ijcai2023p317}, or minimizes exploitability using learned best responses \citep{martin2025approxed}. In contrast, we parameterize each player's evolving mixed strategy with a diffusion model to capture more expressive distributions.

% Finite two-player zero-sum games admit linear-programming formulations \citep{dantzig1951proof}, while equilibrium computation in finite general-sum games is PPAD-complete \citep{chen2009settling}. Continuous action games remain challenging because mixed strategies are distributions over infinite action spaces and best-response computation may require difficult continuous optimizaition. Prior work establishes continuous-action variants of fictitious-play \citep{perkins2014stochastic} and develops strategy-generaion methods for zero-sum and general-sum continuous games \citep{adam2021double, kroupa2023multiple}. Learning-based approaches include DeepFP, which learns generative approximate best responses \citep{kamra2019deep}, randomized policy networks trained with zeroth-order equilibrium dynamics \citep{ijcai2023p317}, and exploitability minimization with learned best responses \citep{martin2025approxed} . Our approach instead parameterizes each player's evolving mixed strategy with a diffusion model and trains it using OMD based equilibrium updates~\citep{rakhlin2013optimization}.

\textbf{Diffusion Models And Its Application in Online RL} Diffusion models learn complex, multimodal distributions through iterative denoising and score matching \citep{ho2020denoising, song2021score}. Their ability to model complex distributions motivated their use as policy parameterizations in RL \citep{wang2024diffusion, psenka2024learning, ren2025diffusion}.  Most closely related to ours, \citet{ma2025efficient} show that exponential reweighting of samples from the current diffusion policy is sufficient to implement
a KL-regularized policy-improvement step. We adapt the reweighting technique and address it in the equilibrium computation setting. We give a comprehensive review of the field in \Cref{appdx:additional-related-work}.

\textbf{Diffusion Applications in Strategic Environment} To our knowledge, the only direct application of diffusion to our setting is \citet{karthikeyan2025difffp}, which we include as a baseline and discuss in \Cref{sec:experiments}. We discuss diffusion model applications in other related strategic settings in \Cref{appdx:additional-related-work}.

% \textcolor{red}{Kai: literature about zero-sum and general-sum games. Why zero-sum with continuous action space is still hard?}

% \textbf{Diffusion and Flow-Based Models for Online RL} Diffusion and flow-based policies have been widely studied for RL \citep{wang2024diffusion, psenka2024learning, ma2025efficient, ren2025diffusion}. Most closely related to our policy update, \citet{ma2025efficient} show that exponential reweighting of samples from the current diffusion policy implements
% a KL-regularized policy-improvement step, yielding Diffusion Policy Mirror Descent.

% Related work also applies expressive diffusion policies in cooperative multi-agent RL~\citep{li2026diffusing}, while flow-matching policies have been developed for online continuous and combinatorial control \citep{park2025flow, zhang2025reinflow, mcallister2026flow, kong2026latent}. These methods optimize return or coordination rather than equilibrium against strategic opponents. In our setting, policy-improvement weights depend on utility against an evolving opponent. We additionally use ELBO differences as approximate density ratios to reuse samples across strategic iterates.
\section{Experiments}\label{sec:experiments}
Beyond the Cournot game, we evaluate \textbf{DiffNash} on a forest security game and a continuous variant of Colonel Blotto, two widely studied games that also admit intuitive visualizations.

\textbf{Metrics and Baselines:}
We use Nash Gap (\Cref{def:nash-gap}) as our main evaluation metric. The best response optimization required in Nash Gap cannot be computed exactly. Following \citet{kamra2019deep}, we use an evolutionary search that repeatedly retains high-utility candidates and perturbs them with Gaussian noise to approximate the Nash Gap. We defer computation details to \Cref{appdx:nash-gap-computation}. 

We compare against three baselines. \textbf{DeepFP} \citep{kamra2019deep} represents each player's strategy in FP with a neural generator $\mathcal T_\theta(\vsample)$, where $\vsample \sim \mathcal N(0,\mathbf{I})$. \textbf{Neural OMD} adapts the mirror descent policy optimization framework of \citet{tomar2020mirror} to equilibrium computation, using a Gaussian neural policy for each player. \textbf{DiffFP} \citep{karthikeyan2025difffp} parameterizes player strategies with diffusion models within FP. Since FP updates do not directly provide samples from the target strategy, DiffFP instead refines sampled actions $\vaction^i$ toward higher-utility regions using $\nabla_{\vaction^i}u^i(\vaction^i;\Pi^{-i,k})$ and trains the diffusion policy on the refined actions. These baselines help isolate the benefits of diffusion policy and principled policy updates. For fair comparison, we train each player’s policy for $40K$ gradient updates across experiments. We use default hyperparameters for baselines.
    
% \begin{itemize}
%     \item \textbf{DeepFP}: \citet{kamra2019deep} uses a neural network $f_\theta(\vsample), \vsample \sim \mathcal N(0,I)$ to represent player policy in FP.

%     \item \textbf{Neural OMD}: \citet{tomar2020mirror} uses a neural network to parameterize a Gaussian policy and performs mirror-descent policy updates in RL. We adapt it to equilibrium computation while retaining its Gaussian parameterization for each player’s policy.
    
%     \item \textbf{DiffFP}: \citet{karthikeyan2025difffp} use diffusion model to parameterize player strategies within FP. Because the FP update cannot be sampled to train diffusion, they adopt a heuristic procedure: sample initial actions $\vaction^i$, guide them toward high-utility regions using $\nabla_{\vaction^i}u^i(\vaction^i;\Pi^{-i,k})$, and train a diffusion model on the resulting actions.
% \end{itemize}

\subsection{Forest Security Game}
\label{sec: forest results}
% We evaluate DiffNash on the continuous forest security game used by \citet{kamra2018policy} and DeepFP \citep{kamra2019deep}. The underlying spatial enforcement problem originates in the forest economics literature \citep{albers2010spatial} and was later studied as a computational security game by \citep{johnson2012patrol}.
Forest security game \citep{johnson2012patrol} originated from forest economics \citep{albers2010spatial}. We consider the continuous version studied by \citet{kamra2018policy} and \textbf{DeepFP} \citep{kamra2019deep}. A defender selects $M$ locations to patrol, while an attacker selects $N$ locations to target in an endangered circular forest. Each location is parameterized by its radial distance and radian angle. Hence, defender and attacker's action dimensionality are $2M$ and $2N$, respectively. As $M$ and $N$ increase, the resulting strategy space becomes high-dimensional, a setting where \citet{kamra2019deep} reports degraded performance. An attacker can collect trees within $R_A$ of its location, and a defender can stop any attacker within $R_D$ of them. The attacker's utility is the value of the harvested trees but receives no value and incurs a penalty if caught by any defender. The defender’s utility is the negative of the attacker’s utility. Across experiments, we keep the interaction and payoff structure fixed and vary only the spatial distribution of trees along with the number of defenders and attackers. Changing the geometry creates qualitatively different optimal strategies by altering which attack locations are valuable, which regions can be protected together, and how many spatially separated patches the strategy must cover. In this game, we aim to find a robust patrol strategy for defender, so besides Nash gap, we also compute $\mathrm{Exploit_D}$, measured as the attacker's best-response reward against the learned average defender. We seek strategies profiles with low $\mathrm{Exploit_D}$ and Low Nash Gap $\mathrm{GAP}$.

% \paragraph{Experimental setup.}
% We compare DiffNash against three baselines chosen to isloate its main components. \textbf{DeepFP} \citep{kamra2019deep} uses generative neural networks to represent approximate best responses within fictitious play, but can struggle to represent sharply multimodal strategies. \textbf{DiffFP} \citep{karthikeyan2025difffp} introduces diffusion policies, but retains fictitious-play updates and local $Q$-guided sampling, which can concentrate on locally discovered modes. We introduce \textbf{Neural OMD}, which uses the same OMD training framework as DiffNash, but replaces the diffusion policy with a squashed diagonal Gaussian. Together, these baselines isolate the effects of both the expressive diffusion policy and the OMD update step. We report the defender exploitability $\epsilon_D^{\mathrm{BR}}$, measured as the attacker's best-response reward against the learned average defender, together with the Nash gap $G$; lower is better for both. Following DeepFP, best responses are approximated with an evolutionary search that repeatedly retains high-utility candidates and perturbs them with Gaussian noise.

\textbf{Experimental results.}
Table~\ref{tab:forest-results} shows that DiffNash achieves the best mean performance across all six forests, reducing Nash gap by an average of $\textbf{35.6\%}$ relative to the strongest baseline. The gap is most pronounced in heterogeneous, multi-resource settings, where equilibrium strategies must distribute probability mass across several strategically distinct regions. DiffNash also achieves broader effective coverage. In F3 of Figure~\ref{fig:forest}, for example, \textbf{DiffNash}'s guards (blue) concentrate around all four tree clusters (green), whereas DeepFP and DiffFP learn more diffuse patrol strategies and Neural OMD places substantial mass in low-tree-density regions. Importantly, diffuse coverage is not necessarily effective coverage: allocating guard probability to regions with few or no trees provides little protection value. We analyze the failure modes of all baselines:
\begin{itemize}
    \item \textbf{DeepFP/NeuralOMD}: Both exhibit mode collapse. In F2 of \Cref{fig:forest}, \textbf{Neural OMD} concentrates on a single forest patch. For \textbf{DeepFP}, \Cref{fig:forest} and additional Appendix \ref{forest-vis} visualizations show both players' strategies forming narrow circular bands. Since actions are represented in polar coordinates, this pattern indicates collapse along the radial dimension.
    \item \textbf{DiffFP}: In large forests, small changes in guard locations are unlikely to change whether attackers are captured. Consequently, the defender utility gradient, i.e., $\nabla_{\vaction^i}u^i(\vaction^i;\Pi^{-i,k})$, is nearly zero for most locations, causing \textbf{DiffFP}’s defender strategy to get stuck, as shown by \Cref{fig:forest}'s F2. In contrast, \textbf{DiffNash} does not rely on local utility gradients. We update the entire defender policy in a principled way, yielding substantially better performance.
\end{itemize}

% illustrates this behavior on two Forest structures. DeepFP often remains overly diffuse, Neural OMD collapses towards a small number of modes, and DiffFP concentrates around locally discovered regions. DiffNash more consistently maintains probability mass across the relevant forest patches, which translates into lower exploitability and Nash gap. Complete visualizations for all six environments are provided in Figure \ref{fig:forest-all} in Appendix \ref{forest-vis}. Appendix \ref{app:general-sum} further evaluates DiffNash on a general-sum variant of the forest game, where the two players have distinct payoff functions.

\begin{table*}[t]
    \centering
    \caption{
        \textbf{Forest security game results.}
        % Comparison with DeepFP \citep{kamra2019deep}, DiffFP
        % \citep{karthikeyan2025difffp}, and Neural OMD, a controlled OMD baseline introduced in this work. 
        We report the attacker best-response reward $\text{Exploit}_D$ against the learned average defender and the two-sided Nash gap $G$ as mean $\pm$ 95\% Student-$t$ confidence intervals over $5$ random seeds; lower is better for both. F1 and F2 are taken from \citet{kamra2019deep}.
    }
    \label{tab:forest-results}

    \setlength{\tabcolsep}{3.5pt}
    \renewcommand{\arraystretch}{1.15}
    \scriptsize

    \resizebox{\textwidth}{!}{
    \begin{tabular}{ll cc cc cc cc}
        \toprule
        & &
        \multicolumn{2}{c}{\textbf{DeepFP}} &
        \multicolumn{2}{c}{\textbf{Neural OMD}} &
        \multicolumn{2}{c}{\textbf{DiffFP}} &
        \multicolumn{2}{c}{\textbf{DiffNash (Ours)}} \\
        \cmidrule(lr){3-4}
        \cmidrule(lr){5-6}
        \cmidrule(lr){7-8}
        \cmidrule(lr){9-10}

        \textbf{Forest} &
        \textbf{$m=n$} &
        $\text{Exploit}_D$ & $\text{Gap}$ &
        $\text{Exploit}_D$ & $\text{Gap}$ &
        $\text{Exploit}_D$ & $\text{Gap}$ &
        $\text{Exploit}_D$ & $\text{Gap}$ \\

        \midrule

        \textsc{F1}
        & 1
        & $244.5 \pm 164.2$
        & $245.8 \pm 160.6$
        & $212.6 \pm 160.4$
        & $216.6 \pm 160.2$
        & $490.2 \pm 3.4$
        & $493.0 \pm 3.6$
        & $\mathbf{96.7 \pm 21.2}$
        & $\mathbf{79.5 \pm 28.8}$ \\

        \textsc{F1}
        & 3
        & $370.5 \pm 309.2$
        & $373.3 \pm 297.8$
        & $307.8 \pm 287.6$
        & $314.5 \pm 290.0$
        & $735.8 \pm 40.8$
        & $741.3 \pm 41.0$
        & $\mathbf{130.4 \pm 9.3}$
        & $\mathbf{122.1 \pm 12.5}$ \\

        \midrule

        F2
        & 1
        & $338.6 \pm 6.5$
        & $330.3 \pm 5.8$
        & $309.7 \pm 28.3$
        & $313.7 \pm 28.4$
        & $354.9 \pm 0.9$
        & $357.8 \pm 0.7$
        & $\mathbf{300.8 \pm 5.2}$
        & $\mathbf{264.4 \pm 95.5}$ \\

        F2
        & 3
        & $702.1 \pm 84.3$
        & $687.4 \pm 78.5$
        & $501.0 \pm 179.6$
        & $500.1 \pm 177.2$
        & $817.3 \pm 101.5$
        & $821.6 \pm 101.3$
        & $\mathbf{322.5 \pm 110.1}$
        & $\mathbf{300.7 \pm 126.4}$ \\

        \midrule

        \textsc{F3}
        & 5
        & $578.8 \pm 109.4$
        & $565.7 \pm 119.7$
        & $767.1 \pm 278.4$
        & $761.1 \pm 271.1$
        & $818.5 \pm 123.7$
        & $800.9 \pm 121.3$
        & $\mathbf{559.7 \pm 76.4}$
        & $\mathbf{528.2 \pm 142.7}$ \\

        \textsc{F4}
        & 5
        & $804.4 \pm 118.5$
        & $776.6 \pm 117.4$
        & $1145.9 \pm 191.6$
        & $1140.9 \pm 187.3$
        & $836.4 \pm 114.7$
        & $822.7 \pm 112.8$
        & $\mathbf{660.4 \pm 39.8}$
        & $\mathbf{622.6 \pm 52.3}$ \\
        \bottomrule
    \end{tabular}
    }
\end{table*}

\begin{figure*}[h]
    \centering
    \includegraphics[width=0.7\textwidth]
        {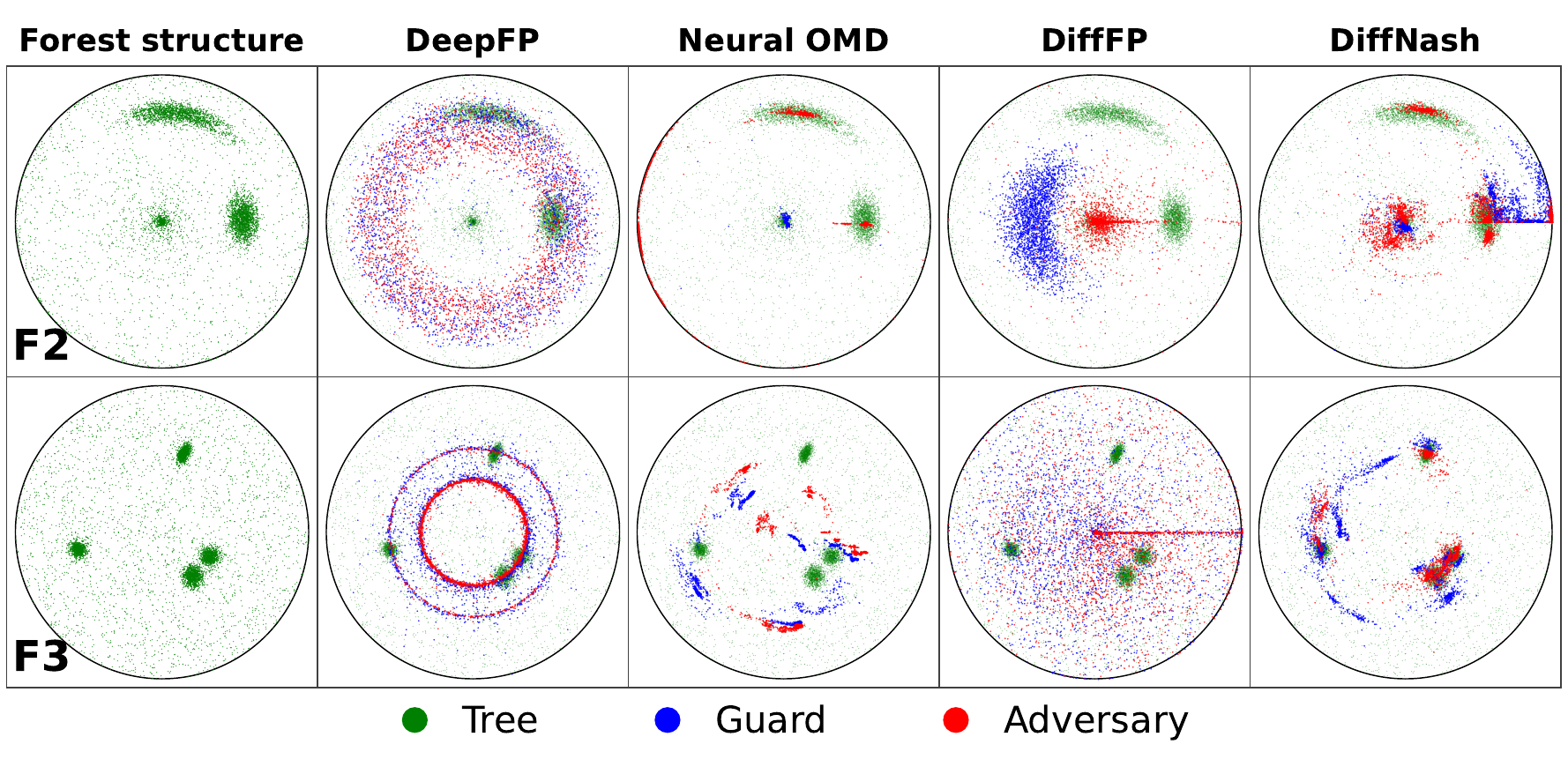}
    \caption{
        \textbf{Empirical strategies in the forest security game} \textit{Top:} F2 $(m = n = 3)$. \textit{Bottom:} F3 $(m = n = 5)$. Green points denote trees. Each blue/red points denote a sample from defender/attacker policy, respectively. For each method and environment, the displayed strategy is taken from $5$ random seeds whose Nash gap is the median. Quantitative results are in Table \ref{tab:forest-results}.
    }
    \label{fig:forest}
\end{figure*}

\subsection{Continuous Colonel Blotto}
\label{sec:results-blotto}

We evaluate \textbf{DiffNash} on the continuous Colonel Blotto game \citep{adam2021double} where two players simultaneously allocate a fixed budget across $n$ weighted battlefields. Player's utility depends on how much they and opponent invest in each battlefield. Because allocations are coupled by the budget constraint, equilibrium play generally requires nontrivial randomization over a continuous action space. We consider games with $n=6$ and $n=10$. \citet{adam2021double} provides a discrete algorithm which approximates the target equilibrium well \footnote{Such oracle-based methods do not directly extend to general continuous-action games without a tractable best-response oracle; see Appendix~\ref{app:blotto-implementation}.}, visualized in the first column of \Cref{fig:blotto6}. We vary player utilty function to get distinct target equilibrium, one with symmetric and another with radial shape. Full game details are provided in Appendix \ref{app:blotto}. 

\begin{wraptable}{r}{0.55\textwidth}
    \caption{
        \textbf{Continuous Colonel Blotto results.}
        Nash gap $G$ on symmetric and radial-band Blotto with $n\in\{6,10\}$.
        Values are mean $\pm$ 95\% Student-$t$ confidence intervals over five seeds; lower is better.
    }
    \label{tab:blotto-results}

    \setlength{\tabcolsep}{3.5pt}
    \renewcommand{\arraystretch}{1.15}
    \scriptsize

    \resizebox{\linewidth}{!}{%
    \begin{tabular}{llcccc}
        \toprule
        $\boldsymbol{n}$ &
        \textbf{Equilibrium} &
        \textbf{DeepFP} &
        \textbf{Neural OMD} &
        \textbf{DiffFP} &
        \textbf{DiffNash (Ours)} \\
        \midrule
        6 & Symmetric
        & $6.854 \pm 1.366$
        & $6.807 \pm 1.832$
        & $\mathbf{0.919 \pm 0.065}$
        & $0.942 \pm 0.168$ \\

        6 & Radial band
        & $0.688 \pm 0.001$
        & $0.836 \pm 0.236$
        & $1.038 \pm 0.056$
        & $\mathbf{0.501 \pm 0.139}$ \\
        \midrule
        10 & Symmetric
        & $9.525 \pm 1.707$
        & $7.233 \pm 1.304$
        & $\mathbf{1.506 \pm 0.056}$
        & $1.724 \pm 0.346$ \\

        10 & Radial band
        & $0.696 \pm 0.002$
        & $0.793 \pm 0.318$
        & $1.446 \pm 0.037$
        & $\mathbf{0.531 \pm 0.116}$ \\
        \bottomrule
    \end{tabular}%
    }
\end{wraptable}
\textbf{Experimental results.} Overall \textbf{DiffNash} reduces the Nash Gap by $\textbf{24.7\%}$ relative to the strongest baseline. In the symmetric setting of \Cref{tab:blotto-results}, diffusion-based policies achieve substantially lower Nash gaps than the neural baselines, highlighting their ability to represent multimodal strategies. In the radial-band setting, \textbf{DiffNash} consistently outperforms \textbf{DiffFP}. As in the forest game, when high-utility actions lie within a narrow region (the radial band), \textbf{DiffFP} can receive little local gradient signal from low-utility initializations. In \Cref{fig:blotto6}, this causes its policy to remain concentrated near the center, outside the support of the target equilibrium, whereas \textbf{DiffNash} recovers the correct equilibrium geometry. We visualize all settings in Appendix \Cref{fig:blotto-all}.

\begin{figure*}[h]
    \centering
    \includegraphics[width=0.7\textwidth]
        {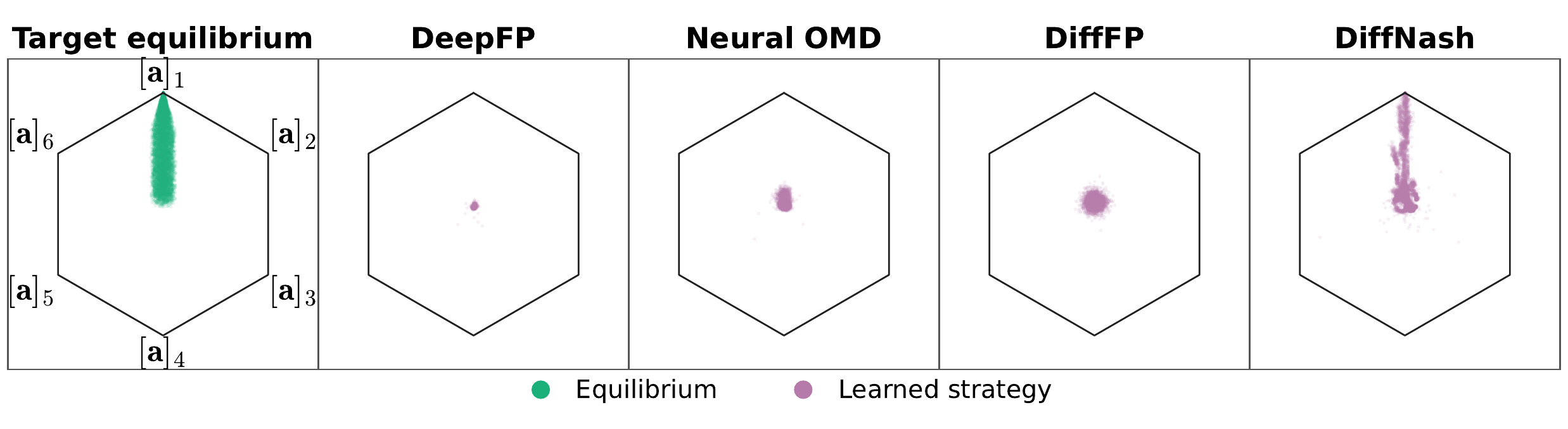}
    \caption{
       Continuous Colonel Blotto (radial-band variant,$n=6$).
       Green denotes samples from the target equilibrium, while purple denotes samples from the learned strategy. Each point represents an allocation, where proximity to $[\vaction]_j$ indicates a larger allocation to the $j$-th battlefield. The point cloud visualizes the distribution of allocations induced by each strategy.
    }
    \label{fig:blotto6}
\end{figure*}
% \begin{figure*}[h]
%     \centering
%     \includegraphics[width=0.7\textwidth]
%         {figures/blotto6.pdf}
%     \captionsetup{width=0.7\textwidth}
%     \caption{
%         \textbf{Continuous Colonel Blotto (radial-band variant, $n=6$).}
%         Green denotes samples from the target equilibrium, while purple denotes samples from the learned strategy. Each point represents an allocation, where proximity to $[\vaction]_j$ indicates a larger allocation to the $j$-th battlefield. The point cloud visualizes the distribution of allocations induced by each strategy.
%     }
%     \label{fig:blotto6}
% \end{figure*}
\begin{wrapfigure}{r}{0.35\columnwidth}
    \centering
   \includegraphics[width=\linewidth]{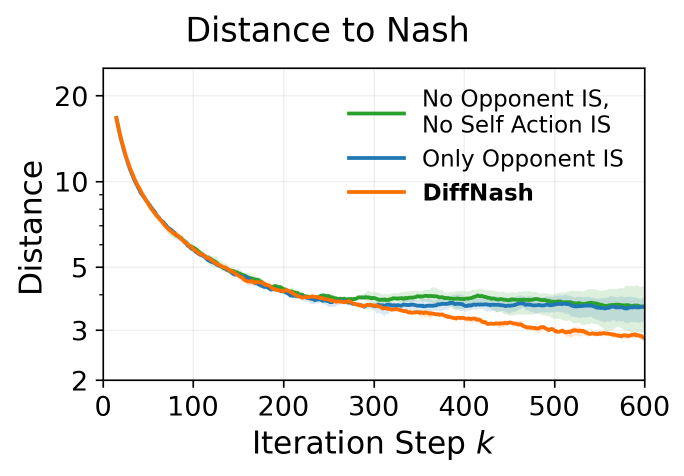}
     \caption{Performance under different IS configurations. Solid lines denote mean, and shades denote the min to max range across runs.}
    \label{fig:cournot-ablation}
\end{wrapfigure}
\textbf{Robustness} Across experiments in \Cref{sec: forest results} and \Cref{sec:results-blotto} we perform minimal hyperparameter tuning: we fix self action ratio $\zeta = 1\%$ and only choose IS trigger $\kappa =\{0,10\%\}$. \textbf{DiffNash} shows strong robustness to hyperparameters.

\subsection{Ablation and Hyperparameter Study}
\label{sec:ablation}

% We ablate the two importance-sampling components of DiffNash on Cournot $d=30$. With Opponent IS, we roughly reduces score network calls to $66\%$ of the No opponent ones. Moreover, equipped with self action IS, \textbf{DiffNash} prevents the early platanue of other baselines without self action IS and shows less variance across runs. A similar pattern has been observed in a forest ablation, which we provide in 
% Appendix~\ref{app:forest-ablation}. Together they justify our IS strategy for both opponent and self action IS. 

\textbf{Ablation} We ablate the two importance-sampling components of \textbf{DiffNash} on Cournot ($d=30$). In \Cref{fig:cournot-ablation}, the green and blue curves almost overlap, indicating that opponent IS has little effect on convergence. Its primary benefit is computational: in this experiment, opponent IS reduces the number of score-network calls to roughly $66\%$ of that required without opponent IS. Moreover, \textbf{DiffNash} (orange) achieves the lowest Nash Gap, 
showing that self action IS prevents the early plateau observed without self action IS and reduces variance across runs. We observe a similar pattern in the forest ablation in Appendix~\ref{app:forest-ablation}. These results support using IS for both opponent and self-action replay.

% DiffNash uses both opponent and own-action IS,
% Only Opponent IS removes own-action replay, and No Opponent IS, No Self Action
% IS removes both. Own-action IS activates when the self-action standard
% deviation falls below $\kappa=10\%$ of its initial value, which occurs around
% $k=100$ in Figure~\ref{fig:cournot-ablation}. After activation, DiffNash
% continues reducing its distance to Nash while both ablated variants plateau,
% supporting the role of own-action replay as a training stabilization
% mechanism. A corresponding forest ablation is provided in
% Appendix~\ref{app:forest-ablation}.

We additionally show our method's performance in a general sum version of the Forest game in \Cref{app:general-sum}. 

\textbf{Hyperparameter Study}
We provide hyperparameter study for self action IS ratio $\zeta$, trigger $\kappa$, self action buffer size and opponent action buffer size in Appendix~\ref{app:hyperparameter-sweep}.

% (Table~\ref{tab:hyperparameter-sweep} and Figure~\ref{fig:is-ratio-sweep-app}).

\section{Conclusion}
We develop \textbf{DiffNash}, a framework that parameterizes player strategies with diffusion models within OMD to compute Nash equilibria in high-dimensional continuous action spaces. We derive a reweighted score-matching objective that trains the diffusion model to realize OMD updates. To improve sample efficiency and training stability, we reuse past actions through a replay buffer and correct the resulting distribution shift with an ELBO-based importance-sampling strategy. On forest security and continuous Colonel Blotto games, \textbf{DiffNash} better recovers equilibrium distributions and reduces Nash gap by $\textbf{35.6\%}$ and $\textbf{24.7\%}$ compared to strongest baselines, respectively.
\newpage

\section*{AI Usage Disclosure}
In this work, we used generative AI tools (GPT-5.6 Sol, GPT-6 Sol, and Claude Opus 5) for polishing and drafting sections of the manuscript, flowcharts, and assists with coding and proofs. Generating synthetic data is not applicable to this work. Additionally, we used generative AI tools for brainstorming research ideas and searching related work. At least one coauthor has reviewed AI-assisted work. We manually check the generated code and proof. The main technical contributions, including reweighted score matching objective and importance sampling strategy, are the authors' original ideas.  We take responsibility for the final content of this work, including text, claims, and artifacts produced with the aid of generative AI.

\section*{Ethics Statement}

This work studies algorithms for computing Nash equilibria in simulated continuous-action games and does not involve human subjects, personal or sensitive data, or real-world interventions. While equilibrium-computation methods may be applicable to domains such as security and resource allocation, all experiments in this work use synthetic benchmark environments and do not involve deployed decision-making systems. We are not aware of additional ethical concerns arising from this work.

\section*{Reproducibility Statement}

We provide complete pseudocode for DiffNash and describe the training procedure, replay mechanism, game definitions, environment parameters, reference solutions, hyperparameters, best-response and Nash-gap evaluation procedures, ablations, and computing infrastructure in the appendix. These details are intended to make our method and reported experiments reproducible from the paper. Some baseline components and environment code, including the original Forest Security Game and DeepFP implementations, were not publicly available and were obtained directly from the original authors. Reported results are averaged over multiple random seeds with confidence intervals, and proofs of all theoretical results are included.

% --------------------------------------------------
% Bibliography
% --------------------------------------------------
\bibliographystyle{plainnat}
\bibliography{references}

% --------------------------------------------------
% Appendix
% --------------------------------------------------
\newpage
\appendix
\appendix
\onecolumn
\begin{center}
	{\Large \textbf{Appendix for DiffNash: Find Nash Equilibrium in High Dimensional Continuous Action Spaces with Diffusion Policies}}
\end{center}
\startcontents[sections]
\printcontents[sections]{l}{1}{\setcounter{tocdepth}{2}}

\section{Limitations}

% Our results suggest several directions for future work. Although replay substantially reduces the need to generate fresh actions, diffusion sampling still requires iterative denoising. Alternative expressive generative policies, such as flow-based models could further reduce the computational cost of equilibrium computation.

We focus on two-player normal-form games. Extending DiffNash to multiplayer or sequential settings, is a natural direction for future work.
\section{Additional Related Work}\label{appdx:additional-related-work}

% \textbf{Diffusion Models} Diffusion models learn complex, multimodal distributions through iterative denoising and score matching \citep{ho2020denoising, song2021score}. Their ability to represent complex distributions has motivated their use as expressive policies in sequential decision making. We use diffusion models to represent mixed strategies over continuous actions.

\textbf{Equilibrium Computation} Finite two-player zero-sum games admit linear-programming formulations \citep{dantzig1951proof}, while equilibrium computation in finite general-sum games is PPAD-complete \citep{chen2009settling}. Continuous action games remain challenging because mixed strategies are distributions over infinite action spaces and best-response computation may require difficult continuous optimizaition. Prior work establishes continuous-action variants of fictitious-play \citep{perkins2014stochastic} and develops strategy-generaion methods for zero-sum and general-sum continuous games \citep{adam2021double, kroupa2023multiple}. Learning-based approaches include DeepFP, which learns generative approximate best responses \citep{kamra2019deep}, randomized policy networks trained with zeroth-order equilibrium dynamics \citep{ijcai2023p317}, and exploitability minimization with learned best responses \citep{martin2025approxed} . Our approach instead parameterizes each player's evolving mixed strategy with a diffusion model and trains it using FP and OMD based equilibrium updates~\citep{rakhlin2013optimization}.
% \textcolor{red}{Kai: literature about zero-sum and general-sum games. Why zero-sum with continuous action space is still hard?}

\textbf{Diffusion and Flow-Based Models for Online RL} Diffusion and flow-based policies have been widely studied for RL \citep{wang2024diffusion, psenka2024learning, ma2025efficient, ren2025diffusion}. Most closely related to our policy update, \citet{ma2025efficient} show that exponential reweighting of samples from the current diffusion policy implements
a KL-regularized policy-improvement step, yielding Diffusion Policy Mirror Descent.

Related work also applies expressive diffusion policies in cooperative multi-agent RL~\citep{li2026diffusing}, while flow-matching policies have been developed for online continuous and combinatorial control \citep{park2025flow, zhang2025reinflow, mcallister2026flow, kong2026latent}. These methods optimize return or coordination rather than equilibrium against strategic opponents. In our setting, policy-improvement weights depend on utility against an evolving opponent. We additionally use ELBO differences as approximate density ratios to reuse samples across strategic iterates.

\paragraph{Diffusion Applications in Strategic Environment} 
Diffusion models and their variants, such as flow matching \citep{lipman2023flow}, have been applied to strategic settings adjacent to equilibrium computation. For auction theory, \citet{wang2026bundleflow} uses flow matching to represent bundle distributions in combinatorial auctions, achieving higher revenue while scaling to larger numbers of items. For green security game, when opponent data is available, diffusion and flow-based models can also  be used to learn opponent behavior. \citet{kong2025robust} uses diffusion models to learn poacher behavior in green security games and design robust patrol strategies. \citet{kong2026generative} uses flow matching to model poacher behavior from real-world data, and \citet{kong2026composite} subsequently uses reinforcement learning to plan patrol policies against the learned poacher behavioral model. Unlike these works, we assume no prior knowledge of opponent behavior and instead seek a Nash strategy, which provides a worst-case guarantee in 2p0s games.
\section{Training Algorithm}\label{app:full-algorithm}

We provide the full pseudocode for DiffNash in
Algorithm~\ref{alg:replay-omd}. Algorithm~\ref{alg:prepare-sample}
constructs replay entries, Algorithm~\ref{alg:compute-weight} computes
their ELBO-based importance weights, and Algorithm~\ref{alg:replay-omd}
gives the complete OMD training procedure. Self-action replay is
activated only if the specified trigger fires; otherwise, training
continues using fresh self-actions.

\begin{algorithm}[ht]
\caption{\textsc{PrepareSample}}
\label{alg:prepare-sample}
\begin{algorithmic}[1]
\Require Action $\vaction_0$, behavior parameters $\theta^{\behavioriteration}$, and diffusion schedule
$\{\alpha_t,\bar{\alpha}_t,\beta_t,\sigma_t^2\}_{t=1}^{T}$.
\Ensure Replay entry $z=(\vaction_0,t,\epsilon,D^{\behavioriteration})$.

\State Sample
\[
t\sim\operatorname{Unif}\{1,\ldots,T\},
\qquad
\epsilon\sim\mathcal N(0,I).
\]

\State Construct
\[
\vaction_t
\leftarrow
\sqrt{\bar{\alpha}_t}\vaction_0
+
\sqrt{1-\bar{\alpha}_t}\epsilon.
\]

\State Compute
\[
D^{\behavioriteration}
\leftarrow
\frac{\beta_t^2}{2\sigma_t^2\alpha_t}
\left\|
s_{\theta^{\behavioriteration}}(\vaction_t,t)
+
\frac{\epsilon}{\sqrt{1-\bar{\alpha}_t}}
\right\|_2^2.
\]

\State \Return $z\leftarrow(\vaction_0,t,\epsilon,D^{\behavioriteration})$.
\end{algorithmic}
\end{algorithm}

\begin{algorithm}[ht]
\caption{\textsc{ComputeWeight}}
\label{alg:compute-weight}
\begin{algorithmic}[1]
\Require Current parameters $\theta^k$, replay entry
$z=(\vaction_0,t,\epsilon,D^{\behavioriteration})$ generated at iteration $\behavioriteration$,
diffusion schedule, and maximum log-weight $C>0$.
\Ensure Clipped importance weight $r^{k,\behavioriteration}\le e^C$.

\State Reconstruct
\[
\vaction_t
\leftarrow
\sqrt{\bar{\alpha}_t}\vaction_0
+
\sqrt{1-\bar{\alpha}_t}\epsilon.
\]

\State Compute
\[
D^k
\leftarrow
\frac{\beta_t^2}
{2\sigma_t^2\alpha_t}
\left\|
s_{\theta^k}(\vaction_t,t)
+
\frac{\epsilon}{\sqrt{1-\bar{\alpha}_t}}
\right\|_2^2.
\]

\State Compute
\[
\widehat{\Delta}^{k,\behavioriteration}
\leftarrow
T(D^{\behavioriteration}-D^k).
\]

\State \Return
\[
r^{k,\behavioriteration}
\leftarrow
\exp\!\left(
\min\{\widehat{\Delta}^{k,\behavioriteration},C\}
\right).
\]
\end{algorithmic}
\end{algorithm}

\begin{algorithm}[ht]
\caption{Replay-Buffer Diffusion Online Mirror Descent (DiffNash)}
\label{alg:replay-omd}
\begin{algorithmic}[1]
\Require Outer iterations $K$, inner steps $H$, self-action batch size
$B_{\mathrm{self}}$, opponent batch size $M$, initial buffer size $N_0$,
opponent-buffer capacity $N_{\mathrm{buf}}$, refresh size $R$, OMD parameter
$\eta$, self-action replay fraction $\zeta$, trigger threshold $\kappa$, and
maximum log-weight $C$.

\State Initialize $\pi_{\theta^{i,1}}=\operatorname{Unif}(\mathcal A^i)$,
$i\in\{1,2\}$; initialize each player's opponent replay buffer and self-action
history with $N_0$ samples using \textsc{PrepareSample} with behavior parameters
$\theta^{i,1}$, and record the initial self-action spread.

\For{$k=1,\ldots,K$}
    \For{$i\in\{1,2\}$}

        \State Sample opponent entries $\widetilde z^{-i,(1:L)}$ with behavior
        iterations $k^{(1)},\ldots,k^{(L)}$; let $\widetilde \vaction_0^{-i,(\ell)}$ denote the
        clean action in $\widetilde z^{-i,(\ell)}$, and compute
        \[
        \widetilde r^{-i,k,\behavioriteration}
        \leftarrow
        \Call{ComputeWeight}{\theta^{-i,k},\widetilde z^{-i, \behavioriteration},C},
        \qquad
        \widetilde w_{(\ell)}^{-i}
        \leftarrow
        \frac{L\widetilde r^{-i,k,\behavioriteration}} {\sum_{l=1}^L\widetilde r^{-i,k,\behavioriteration}},
        \quad l\in[L].
        \]

        \State Set
        $b_{\mathrm{IS}}\leftarrow\lceil\zeta B_{\mathrm{self}}\rceil$
        if the current self-action spread is below $\kappa$ times its initial
        value, and $b_{\mathrm{IS}}\leftarrow0$ otherwise; set
        $\theta^{i,k,0}\leftarrow\theta^{i,k}$.

        \For{$h=1,\ldots,H$}

            \State Sample
            \[
            \vaction_0^{i,(b_{\mathrm{IS}}+1:B_{\mathrm{self}})}
            \overset{\mathrm{i.i.d.}}{\sim}\pi_{\theta^{i,k}},
            \qquad
            z^{i,(1:b_{\mathrm{IS}})}
            \sim \text{self-action history}.
            \]

            \State For $m \in[b_{\mathrm{IS}}]$, let $\vaction_0^{i,\selfbehavioriteration}$ be the clean
            action in $z^{i,\selfbehavioriteration}$, with behavior iteration $\selfbehavioriteration$, and compute
            \[
            r^{i,k, \selfbehavioriteration}
            \leftarrow
            \Call{ComputeWeight}{\theta^{i,k},z^{i,\selfbehavioriteration},C},
            \qquad
            w_{(m)}^{i}
            \leftarrow
            \frac{b_{\mathrm{IS}}r^{i,k,\selfbehavioriteration}}
            {\sum_{j=1}^{b_{\mathrm{IS}}}r^{i,k,\selfbehavioriteration}}.
            \]

            \State For $m \in[B_{\mathrm{self}}]$, estimate
            \[
            \widehat u^i
            \left(\vaction_{(m)}^{i};\pi_{\theta^{-i,k}}\right)
            \leftarrow
            \frac1 L\sum_{\ell=1}^L
            \widetilde w_{(\ell)}^{-i}
            u^i\left(
            \vaction_{(m)}^{i},\widetilde \vaction_{(\ell)}^{-i,\behavioriteration}
            \right).
            \]

            \State For $m \in[B_{\mathrm{self}}]$, sample
            $t \sim\operatorname{Unif}\{1,\ldots,T\}$ and
            $\epsilon \sim\mathcal N(0, I)$, and set
            \[
            \vaction_{t, (m)}^{i}
            \leftarrow
            \sqrt{\bar\alpha_{t}}\,\vaction_{(m)}^{i}
            +
            \sqrt{1-\bar\alpha_{t}}\,
            \epsilon.
            \]

            \State Take one optimizer step from $\theta^{i,k,h-1}$ on
            \[
            \begin{aligned}
            \mathcal L^{i,k,h}(\theta)
            =\frac1{B_{\mathrm{self}}}\Bigg[
            &\sum_{m=1}^{b_{\mathrm{IS}}}
            w_{(m)}^{i}
            \exp\!\left(
            \frac{\widehat u^i
            (\vaction_{(m)}^{i,\selfbehavioriteration};\pi_{\theta^{-i,k}})}{\eta}
            \right)
            (1-\bar\alpha_{t})
            \left\|
            s_\theta
            (\vaction_{t,(m)}^{i,\selfbehavioriteration},t)
            +
            \frac{\epsilon}
            {\sqrt{1-\bar\alpha_{t}}}
            \right\|_2^2
            \\
            +{}&
            \sum_{m=b_{\mathrm{IS}}+1}^{B_{\mathrm{self}}}
            \exp\!\left(
            \frac{\widehat u^i
            (\vaction_{(m)}^{i, k};\pi_{\theta^{-i,k}})}{\eta}
            \right)
            (1-\bar\alpha_{t})
            \left\|
            s_\theta
            (\vaction_{t,(m)}^{i,k},t)
            +
            \frac{\epsilon}
            {\sqrt{1-\bar\alpha_{t}}}
            \right\|_2^2
            \Bigg],
            \end{aligned}
            \]
            obtaining $\theta^{i,k,h}$.

        \EndFor

        \State Set
        $\theta^{i,k+1}\leftarrow\theta^{i,k,H}$.

        \State Draw $R$ actions from $\pi_{\theta^{i,k+1}}$ and apply
        \textsc{PrepareSample} with behavior parameters $\theta^{i,k+1}$;
        append the resulting entries to both buffers and retain the most recent
        $N_{\mathrm{buf}}$ opponent entries.

    \EndFor
\EndFor

\State \Return
$\displaystyle
\Pi^{i,K}
=
\frac1K\sum_{k=1}^K\pi_{\theta^{i,k}},
\quad i\in\{1,2\}$.
\end{algorithmic}
\end{algorithm}

\newpage 
\section{Details about Experiment Setup}

\subsection{Nash Gap Computation}\label{appdx:nash-gap-computation}

For a learned strategy profile $\pi=(\pi^1,\pi^2)$, we measure the total gain available from unilateral deviations. For player $i$, define
\begin{equation}
    \epsilon^i(\pi)
    =
    \max_{a^i\in\mathcal{A}^i}
    u^i(a^i;\pi^{-i})
    -
    U^i(\pi^i,\pi^{-i}),
\end{equation}
and compute the Nash gap as
\begin{equation}
    G(\pi^1,\pi^2)
    =
    \epsilon^1(\pi)+\epsilon^2(\pi).
    \label{eq:nash-gap-eval}
\end{equation}
Because expected utility is linear in the deviating strategy, it is sufficient to optimize over pure best responses. In the two-player zero-sum setting, writing $U$ for player 1's payoff, this simplifies to
\begin{equation}
    G(\pi^1,\pi^2)
    =
    \max_{a^1\in\mathcal{A}^1} U(a^1,\pi^2)
    -
    \min_{a^2\in\mathcal{A}^2} U(\pi^1,a^2).
\end{equation}

For continuous action spaces, the best-response optimization is computed with a game-specific evaluation oracle. For the forest game, following \citet{kamra2019deep}, we use evolutionary search with a population of 50 candidates for 20 iterations. At each iteration, the top 10 candidates are retained and four Gaussian-perturbed mutations are generated from each survivor. Candidate utilities are estimated against 100 opponent samples during the search, and the resulting best response is reevaluated against 1,000 samples. When estimating the individual unilateral gains in \eqref{eq:nash-gap-eval}, the utility of the learned profile is estimated from 2,000 sampled action pairs.

For symmetric Colonel Blotto, we follow the discretized best-response procedure of \citet{adam2021double}. Final Nash gaps use a simplex grid with spacing $1/1024$ and an exact knapsack dynamic program, with 20,000 samples from each learned policy. For the radial-band variant, we approximate each best response from a pool of approximately 11,400 candidates, refine the best 64 candidates for 400 gradient steps, and estimate final payoffs using 8,000 policy samples per player. The same evaluation procedure is used for all methods within each game.

\subsection{Experimental Details of Cournot Game}
\label{app:cournot}

We consider a two-player multi-market Cournot competition game with $d$ independent markets, where $A>0$ is the demand intercept and $b>0$ controls the slope of inverse demand. Player $i\in\{1,2\}$ chooses production quantities across $d$ markets:

\[
\vaction^i=([\vaction^i]_1,\ldots,[\vaction^i]_d)\in[0,q_{\max}]^d.
\]

In each market $j$, the inverse demand function is

\[
p_j(\vaction^1_j,\vaction^2_j)
=
A-b(\vaction^1_j+\vaction^2_j),
\]

and both firms have constant marginal cost $c$. Player $i$'s total utility is

\begin{equation}
u^i(\vaction^i,\vaction^{-i})
=
\sum_{j=1}^{d}
\left(
A-b(\vaction^i_j+\vaction^{-i}_j)-c
\right)\vaction^i_j.
\label{eq:cournot-payoff}
\end{equation}

We use $A=10$, $b=1$, $c=1$, and
$q_{\max}=(A-c)/b=9$, with $d\in\{1,30\}$ in the experiments.

In the implementation, the policy produces a normalized action
$\mathbf{z}^i\in[-1,1]^d$, which is mapped to the feasible Cournot action by

\[
\vaction^i
=
\frac{\mathbf{z}^i+1}{2}q_{\max}.
\]

During optimization, we use the scaled utility $u^i/d$ so that its magnitude
does not grow with the number of markets.

\paragraph{Analytical equilibrium.}
The separability and quadratic structure of the Cournot game allow its Nash
equilibrium to be characterized exactly.

\begin{proposition}
\label{prop:cournot-equilibrium}
The Cournot game in~\eqref{eq:cournot-payoff} has a unique Nash equilibrium $(\vaction^{1,\star}, \vaction^{2, \star})$, and for every market $j=1,\ldots,d$,
\[
[\vaction^{1,\star}]_j
=
[\vaction^{2,\star}]_j
=
\frac{A-c}{3b}.
\]
Under our parameters, $[\vaction^{1,\star}]_j=[\vaction^{2,\star}]_j=3$ for every market.
\end{proposition}

\begin{proof}
Consider an arbitrary mixed strategy $\pi^{-i}$ of player $-i$, and let
\[
\mu^{-i}_j
=
\mathbb{E}_{\vaction^{-i}\sim\pi^{-i}}[[\vaction^{-i}]_j]
\]
denote its mean production in market $j$. Since the payoff is separable across
markets and linear in the opponent's production, the expected utility from a
action $\vaction^i$ is
\[
\mathbb{E}_{\vaction^{-i}\sim\pi^{-i}}
\left[
u^i(\vaction^i,\vaction^{-i})
\right]
=
\sum_{j=1}^{d}
\left[
(A-c-b\mu^{-i}_j)\vaction^i_j
-
b(\vaction^i_j)^2
\right].
\]

For $b>0$, this expression is strictly concave in $\vaction^i$, so player $i$ has a
unique best response against any mixed opponent strategy. In each market,
\[
\operatorname{BR}^i_j(\pi^{-i})
=
\frac{A-c-b\mu^{-i}_j}{2b}.
\]

Since
\[
0\leq \mu^{-i}_j\leq q_{\max}=\frac{A-c}{b},
\]
this quantity lies in $[0,q_{\max}]$, so the action constraints do not alter
the best response.

A player can randomize at a Nash equilibrium only over best responses. Since
the best response is unique, every Nash equilibrium must therefore be pure (all probability mass on a single action).
For each market $j$, the equilibrium quantities satisfy
\[
2b [\vaction^{1,\star}]_j+b [\vaction^{2,\star}]_j=A-c,
\qquad
b [\vaction^{1,\star}]_j+2b [\vaction^{2,\star}]_j=A-c.
\]

Subtracting the two equations gives
$[\vaction^{1,\star}]_j=[\vaction^{2,\star}]_j$, and substitution yields
\[
[\vaction^{1,\star}]_j=[\vaction^{2,\star}]_j
=
\frac{A-c}{3b}.
\]

Because the markets are independent, this holds for every
$j=1,\ldots,d$. With $A=10$, $b=1$, and $c=1$, the equilibrium quantity is
$3$ in every market.
\end{proof}

\subsection{Experimental Details of Forest Protection Game}
\label{app:forest}

\subsubsection{Environment Details}
\label{app:forest-environment}

We use the continuous forest security game introduced by
\citet{kamra2018policy} and subsequently adopted as the forest protection benchmark in DeepFP \citep{kamra2019deep}. Our environment and parameter
settings follow the latter. The game is played between a defender $D$ and an adversary $A$ in a circular forest containing a fixed set of tree locations
$\mathcal{T}\subset\mathbb{R}^{2}$. The adversary controls $N$ adversaries, while the defender controls $M$ stationary guards.

Each adversary enters at the boundary of the forest, travels along a radial path toward the center, stops at a chosen location, and exits along the same path. Its stopping location therefore determines both its trajectory and its harvesting region. The adversary's action is

\begin{equation*}
    \vaction_A=([\vaction_A]_1,\ldots,[\vaction_A]_N)\in\mathcal{F}^{N},
    \label{eq:forest-adversary-action}
\end{equation*}

where $\mathcal{F}$ is the unit disk and $[\vaction_A]_k$ is the stopping location of the $k$-th adversary. The defender selects the guard locations

\begin{equation*}
    \vaction_D=([\vaction_D]_1,\ldots,
    [\vaction_D]_M)\in\mathcal{F}^{M}.
    \label{eq:forest-defender-action}
\end{equation*}

where $[\vaction_D]_j$ is the location of guard $j$. Locations are represented using normalized polar coordinates $(r,\theta)$, so the adversary and defender action dimensions are $2N$ and $2M$, respectively.

An adversary is captured when its path passes within guard radius $R_g$ of at least one defender resource. Let $P_k([\vaction_A]_k)$ denote the radial path followed by the adversary $k$.  Its capture indicator is

\begin{equation*}
    C_k(\vaction_D, \vaction_A)
    =
    \mathbb{I}
    \left[
        \min_{j\in\{1,\ldots,M\}}
        \min_{p\in P_k([\vaction_A]_k)}
        \lVert p- [\vaction_D]_j\rVert_2
        \leq R_D
    \right].
    \label{eq:forest-capture}
\end{equation*}

An uncaptured adversary harvests trees lying within chopping radius $R_A$ of its stopping location. Define

\begin{equation*}
    H_k([\vaction_A]_k)
    =
    \left\{
        t\in\mathcal{T}:
        \lVert t-[\vaction_A]_k\rVert_2\leq R_A
    \right\}.
    \label{eq:forest-harvest-set}
\end{equation*}

Let $N_{\mathrm{stolen}}(\vaction_D,\vaction_A)$ denote the number of trees successfully removed by the uncaptured adversaries, with overlapping harvest regions combined according to the game simulator. A captured adversary loses its harvest and incurs a penalty $r_{\mathrm{pen}}>0$. The adversary's reward is

\begin{equation*}
    u_A(a_D,a_A)
    =
    N_{\mathrm{stolen}}(a_D,a_A)
    -r_{\mathrm{pen}}
    \sum_{k=1}^{N}C_k(a_D,a_A),
    \label{eq:forest-payoff}
\end{equation*}

and the defender receives the negative reward,
$u_D(a_D,a_A)=-u_A(a_D,a_A)$. Following the DeepFP experimental setup \citep{kamra2019deep}, we use

\begin{equation*}
    R_D=0.1,
    \qquad
    R_A=0.04,
    \qquad
    r_{\mathrm{pen}}=4.0.
    \label{eq:forest-parameters}
\end{equation*}

The reward landscape is highly nonconvex and contains large flat regions. In particular, moving a guard without changing which trajectories it intercepts
does not change the defender's reward. Heterogeneous tree densities also create spatially separated high-value regions, so effective mixed strategies may place probability mass on several distinct guard or adversary configurations. The game therefore tests both the expressiveness of the learned strategy distribution and its ability to improve when informative utility gradients are sparse.

\paragraph{Forest instances.} 

We use four forest layouts, denoted F1--F4 throughout the paper. F1 and F2 contain 4,000 and 6,000 trees, respectively, and are evaluated with $M=N\in\{1,3\}$. F3 and F4 are denser layouts containing 6,200 and 7,100 trees and are evaluated with $M=N=5$. Thus, the six configurations reported in the main experiments are

\[
(\mathrm{F1},1),\;(\mathrm{F1},3),\;
(\mathrm{F2},1),\;(\mathrm{F2},3),\;
(\mathrm{F3},5),\;(\mathrm{F4},5).
\]

The interaction and payoff structure is held fixed across instances; only the spatial tree layout and number of resources change. This changes which attack locations are valuable and how many spatially separated regions must be represented by the equilibrium strategy.

\paragraph{Scaling the action space.}

Each guard and adversary is represented by a polar coordinate $(r,\theta)$, so the defender and attacker action spaces have dimensions $2M$ and $2N$, respectively. Latent actions in $[-1,1]$ are mapped to $r\in[0,1]$ and $\theta\in[0,2\pi]$. Consequently, the headline experiments increase each player's action dimension from 2 at $M=N=1$, to 6 at
$M=N=3$, and 10 at $M=N=5$. Although resources of the same type are interchangeable in the payoff, we use the direct ordered representation and do not explicitly enforce permutation symmetry.

\subsubsection{Implementation Details}
\label{app:forest-implementation}

\paragraph{Compared methods.}

We compare DiffNash with DeepFP \citep{kamra2019deep}, DiffFP \citep{karthikeyan2025difffp}, and Neural OMD. DeepFP uses neural fictitious play with a learned attacker and geometric defender best response. DiffFP replaces the neural policy with a diffusion best response while retaining fictitious-play updates. Neural OMD uses the same OMD training loop as DiffNash but represents the policy with a squashed diagonal Gaussian, isolating the contribution of the diffusion policy. DiffNash combines a diffusion policy with OMD, opponent replay, and self-action importance sampling.

\paragraph{Training details.}

All headline methods are trained for 40,000 gradient updates per player. DiffNash and Neural OMD use 2,000 outer OMD iterations with 20 inner updates per iteration. DiffFP uses 100 fictitious-play iterations with 400 best-response updates per iteration. Unless otherwise noted, we run five seeds. OMD methods use batches of 256 player actions and 128 opponent replay samples. DiffNash maintains a FIFO opponent replay buffer of size 1,024 and draws 1\% of each player batch from its full self-action history.

DiffNash and Neural OMD use Adam \cite{kingma2014adam} with learning rate $3\times10^{-4}$, decayed to $3\times10^{-5}$ during training. The diffusion policies use $T=20$ denoising steps. All experiments are run on CPU using 8 cores and 16\,GB of memory.

\paragraph{Evaluation.}

We report defender exploitability, defined as the best attacker reward against the learned time-averaged defender, together with the two sided Nash gap. Best responses are approximated using evolutionary search on the reference simulator. The search uses a population of 50 candidates for 20 iterations, retaining the top 10 and generating four mutations per survivor. Each candidate is evaluated against 100 sampled opponent actions, followed by a final evaluation against 1,000 samples. Self-play utility for the Nash gap is estimated from 2,000 sampled action pairs.

We evaluate only the final checkpoint and use the complete time-averaged strategy accumulated during training. Reported uncertainty is the 95\% Student-$t$ confidence interval across seeds.

\paragraph{Hyperparameters.}

DiffNash uses a DACER-style denoiser with three 256-unit hidden layers and Mish activations. The diffusion-time embedding is a 16-dimensional sinusoidal encoding followed by a two-layer MLP. We use OMD step size $\eta=1$, normalize utility-derived energies within each batch, clip their logits to $[-3,3]$, and renormalize the resulting weights to mean one. Opponent importance ratios are upper-clipped at 10. Self-action replay constitutes 1\% of the training batch and is activated once the standard deviation of the current action samples falls below 10\% of its initial value.

\subsection{Experimental Details of Continuous Colonel Blotto}
\label{app:blotto}

\subsubsection{Environment Details}
\label{app:blotto-environment}

We follow the continuous Colonel Blotto formulation of
\citet{adam2021double}. Two players simultaneously allocate a unit budget
across $N$ battlefields. Both players have the action space
\begin{equation*}
\mathcal{A}^1=\mathcal{A}^2=\mathcal{A}
    =
    \left\{
        \vaction=([\vaction]_1,\ldots, [\vaction]_N)\in\mathbb{R}_{+}^{n}
        :
        \sum_{j=1}^{N}[\vaction]_j=1
    \right\}.
    \label{eq:blotto-action-space}
\end{equation*}

Let $\vaction^1,\vaction^2\in\mathcal{A}$ denote the allocations selected by players
$1$ and $2$, respectively. The symmetric case and radial band case of \Cref{tab:blotto-results} differ in the utility function of the agent, which we detail below. 

\paragraph{Symmetric}
To construct the symmetric case of \Cref{tab:blotto-results}, we define the following battlefield contest function:

% each battlefield $j$ has a positive weight $\omega_j$. The battlefield contest function is

\begin{equation*}
    \ell_c(z)
    =
    \begin{cases}
        -1, & z\leq -c,\\[1mm]
        z/c, & -c<z<c,\\[1mm]
        1, & z\geq c,
    \end{cases}
    \label{eq:blotto-contest}
\end{equation*}
, where $\omega_j$ is battlefield $j$'s positive weight and $c>0$ controls the allocation advantage required to obtain the full value of a battlefield. Player $1$'s utility is
\begin{equation*}
    u^1(\vaction^1, \vaction^2)
    =
    \sum_{j=1}^{n}
    \omega_j\,
    \ell_c(\vaction^1_j-\vaction^2_j),
    \label{eq:blotto-payoff}
\end{equation*}
and player $2$ receives
\[
    u^2(\vaction^1,\vaction^2)=-u^1(\vaction^1, \vaction^2).
\]
Thus, player $1$ receives $\omega_j$ when it allocates at least $c$ more than player $2$ to battlefield $j$, receives $-\omega_j$ when it allocates at least $c$ less, and receives a payoff that varies linearly with the allocation difference inside this margin. Our symmetric Blotto experiment uses the symmetric N-battlefield setting
\begin{equation*}
    \omega_j=1 \ \text{for all } j,
    \qquad
    c=\frac{1}{16}.
    \label{eq:blotto-primary-setting}
\end{equation*}
for $N = \{6, 10\}$. As $c\rightarrow 0$, $\ell_c$ approaches the sign function used in the classical continuous Colonel Blotto game.

% \paragraph{Standard six-battlefield instance.}
% Our main Blotto experiment uses the symmetric six-battlefield setting
% \begin{equation*}
%     N=6,
%     \qquad
%     \omega_j=1 \ \text{for all } j,
%     \qquad
%     c=\frac{1}{16}.
%     \label{eq:blotto-primary-setting}
% \end{equation*}

The game is symmetric and zero-sum, so its equilibrium value is zero. Because the budget constraint couples all six coordinates, increasing the allocation to one battlefield reduces the allocation available to the others. Equilibrium play therefore requires coordinated allocations rather than independent decisions at individual battlefields. This makes the game a useful test of whether a policy class can represent a broad mixed strategy.

\paragraph{Radial-band instances.}

% We construct two synthetic continuous allocation games with $N \in\{6,10\}$ battlefields. 

To construct the radial band case of \Cref{tab:blotto-results}, we first specify precisely what the radial band is. Define the uniform allocation $\bar{\vaction}$ which allocates $\frac{1}{N}$ budget to each battlefield, and define $\mathbf{e}_1$ be the allocation which allocates all battles to the first battlefield. Now define:
$v_{\mathrm{rad}} = \mathbf{e}_1 - \bar{\vaction}$, center line $\gamma_{\mathrm{rad}}(\tau) = \bar{\vaction} + \tau v_{\mathrm{rad}}$, and width $\omega_{\mathrm{rad}}(\tau) = \min \{0.06, 0.20(1-\tau)\}$. The radial interval is $[\tau_{\min}, \tau_{\max}]$, with 
\[
\tau_{\min}
=
\begin{cases}
0.12, & n=6,\\
0.15, & n=10,
\end{cases}
\]
and $\tau_{\max} = 0.995$. $\gamma_{\mathrm{rad}}(\tau)$ and $\omega_{\mathrm{rad}}(\tau)$ together define the shape of the radial band.

The next step we define how probability is distributed within the band. We put $60\%$ of the probability along the interior portion and $40\%$ closer to the vertex. To achieve that we first sample
First sample

\[
\tau \sim
\begin{cases}
\operatorname{Unif}[\tau_{\min},0.80), & \text{with probability }0.6,\\
\operatorname{Unif}[0.80,\tau_{\max}], & \text{with probability }0.4.
\end{cases}
\]
We also add some perpendicular displacement to make the distribution has width. Formally, given $\tau$, draw the displacement $\xi$ uniformly from
\[
\left\{
\xi \in \mathbb{R}^n :
\mathbf{1}^{\top}\xi = 0,\quad
v_{\mathrm{rad}}^{\top}\xi = 0,\quad
\|\xi\|_2 \le w_{\mathrm{rad}}(\tau),\quad
\gamma_{\mathrm{rad}}(\tau) + \xi \in \Delta^{N-1}
\right\}.
\]
The first constraint make sure the resulting allocation does not exceed the budget, and the second make sure the displacement is orthogonal to the radial band. Define the target distribution $\mu_{\mathrm{rad}}$ as the law of $\gamma_{\mathrm{rad}} + \xi$.

Before we specify the payoff which makes $\mu_{\mathrm{rad}}$ the equilibrium strategy, we first define the following notations. For any allocation, we first locate its clipped projection onto the centerline, measure its perpendicular distance, and finally penalize the portion that exceeds the allowed width. Define the projected radial coordinate $\widehat{\tau}(\vaction)$
and penalty $V_{\mathrm{rad}}(\vaction)$ by
\[
\widehat{\tau}(\vaction)
=
\min\left\{
\tau_{\max},
\max\left\{
\tau_{\min},
\frac{\langle \vaction-\bar \vaction,\,v_{\mathrm{rad}}\rangle}
{\|v_{\mathrm{rad}}\|_2^2}
\right\}
\right\},
\]
\[
V_{\mathrm{rad}}(\vaction)
=
-15\left[
\left\|\vaction-\gamma_{\mathrm{rad}}\bigl(\widehat{\tau}(\vaction)\bigr)\right\|_2
-
w_{\mathrm{rad}}\bigl(\widehat{\tau}(\vaction)\bigr)
\right]_+.
\]
Here $V_{\mathrm{rad}} \leq 0$, and every target sample does not receive penalty. We also define the following Gaussian kernel $K(\cdot, \cdot)$ which measure the similarity between allocations:
\[
K(a^1,a^2)
=
\exp\left(
-\frac{\|a^1-a^2\|_2^2}{2\sigma_K^2}
\right).
\]
, with bandwidth $\sigma_K = 0.08$. We define a centered version of the Gaussian kernel where we subtract each allocation's average similarity to the target radial shape $\mu_\mathrm{rad}$, then add back the target's average similarity to itself. Formally, for independent samples $\radialsample, \radialsample' \sim \mu_\mathrm{rad}$, define
\[
\bar{K}(\vaction^1,\vaction^2)
=
K(\vaction^1, \vaction^2) - \mathbb{E}_{\radialsample} K(\vaction^1, \radialsample)- \mathbb{E}_{b} K(\radialsample, \vaction^2) + \mathbb{E}_{\radialsample, \radialsample'}K(\radialsample,\radialsample').
\]
Set player $1$'s utilities to 
$u^1(\vaction^1, \vaction^2) = \bar{K}_{\mathrm{rad}}(\vaction^1, \vaction^2) + V_{\mathrm{rad}}(\vaction^1) - V_{\mathrm{rad}}(\vaction^2)$ and set $u^2(\vaction^1, \vaction^2) = - u^1(\vaction^1, \vaction^2)$.

\begin{proposition}[Radial Band is a Nash equilibrium]
    $(\mu_{\mathrm{rad}}, \mu_{\mathrm{rad}})$ is the Nash equilibrium of the game we described above. 
\end{proposition}
\begin{proof}
   Observe that for every
$\vaction^1,\vaction^2\in\Delta^{N-1}$,
\[
\mathbb{E}_{\radialsample \sim\mu_{\mathrm{rad}}}
\left[\bar K(\vaction^1, \radialsample)\right]
=0,
\qquad
\mathbb{E}_{\radialsample \sim\mu_{\mathrm{rad}}}
\left[\bar K(\radialsample,\vaction^2)\right]
=0,
\]
by the definition of the centered kernel.
Since $\mu_{\mathrm{rad}}$ is supported inside the corridor,
$V_{\mathrm{rad}}(\radialsample)=0$ almost surely for
$\radialsample \sim\mu_{\mathrm{rad}}$, while
$V_{\mathrm{rad}}(\vaction)\leq0$ for every feasible action.

Let $U^1$ denote the expected utility induced by $u^1$.
For any pure deviations $\vaction^1$ and $\vaction^2$, we therefore have
\[
U^1(\vaction^1,\mu_{\mathrm{rad}})
=
V_{\mathrm{rad}}(\vaction^1)
\leq0,
\qquad
U^1(\mu_{\mathrm{rad}},\vaction^2)
=
-V_{\mathrm{rad}}(\vaction^2)
\geq0.
\]
Moreover,
\[
U^1(\mu_{\mathrm{rad}},\mu_{\mathrm{rad}})=0.
\]
Thus,
\[
U^1(\vaction^1,\mu_{\mathrm{rad}})
\leq
U^1(\mu_{\mathrm{rad}},\mu_{\mathrm{rad}})
\leq
U^1(\mu_{\mathrm{rad}},\vaction^2)
\]
for all $\vaction^1,\vaction^2\in\Delta^{n-1}$.
These inequalities extend to mixed deviations by linearity
of expected utility. Hence
$(\mu_{\mathrm{rad}},\mu_{\mathrm{rad}})$ is an exact Nash
equilibrium with value zero for both $N=6$ and $N=10$.
\end{proof}

In implementation, we replace $\mu_{\mathrm{rad}}^\star$ by a fixed empirical distribution and compute all kernel-centering terms from the same bank. The same argument then applies exactly to the resulting finite game.

\paragraph{Visualization.}
For visualization, six-dimensional allocations are projected onto a regular
hexagon using barycentric coordinates. Let
$[\vaction]_1,\ldots,[\vaction]_N \in\mathbb{R}^N$ denote the hexagon vertices. An allocation
$\vaction$ is plotted at

\begin{equation}
    p(\vaction)
    =
    \sum_{j=1}^{N} [\vaction]_j e_j.
\end{equation}

Under this projection, the uniform allocation appears at the hexagon center.
Allocating the entire budget to battlefield $j$ would map directly to vertex
$e_j$. The same projection is used for both the standard Blotto and radial-band
games, which allows their learned strategy distributions to be compared
directly.

\subsubsection{Implementation Details}
\label{app:blotto-implementation}

\paragraph{Reference solutions.}

For the symmetric Colonel Blotto instances, we compute reference equilibria with the grid-based double-oracle (DO) solver proposed in \citet{adam2021double}. For both $N=6$ and $N=10$, we discretize allocations with grid spacing $1/256$ and run DO to tolerance $10^{-6}$, using an exact knapsack dynamic program to compute best responses over the discretized action space. For $N=6$, DO terminates after $21$ iterations with a $42$-action support and Nash gap $4.31\times 10^{-14}$; for $N=10$, it terminates after $12$ iterations with a $50$-action support and gap $1.33\times 10^{-14}$. Both instances use $c=1/16$. This solver relies on a game-specific best-response oracle: discretization makes exact best responses tractable for these Blotto instances, but such an oracle is generally unavailable in high-dimensional continuous-action games, where best-response optimization may itself be difficult and uniform discretization scales poorly with dimension. We therefore use DO only to construct reference equilibria for these structured symmetric instances, rather than as a general-purpose baseline.

For the radial-band instances, the reference equilibrium is known by construction. We instantiate it using a fixed empirical distribution $\widehat{\mu}^{\star}_n$ with $K=1200$ anchors, consisting of $720$ samples from the interior portion of the corridor and $480$ from the near-extreme portion. The anchor bank is sampled once for each game instance and then frozen, and all kernel-centering terms are computed from the same bank. Thus, $\widehat{\mu}^{\star}_n$ is an exact equilibrium of the implemented finite-anchor game. We use $r_{\min}=0.12$ for $N=6$ and $r_{\min}=0.15$ for $N=10$, with $r_{\mathrm{split}}=0.80$ and $r_{\max}=0.995$. The corresponding reference Nash gap is exactly zero. We additionally report the finite-sample evaluation floors, $0.014213$ for $N=6$ and $0.024741$ for $N=10$.

\paragraph{Compared Methods.}

We compare DiffNash against \textbf{DeepFP}, \textbf{DiffFP}, and \textbf{Neural OMD}. For DeepFP, we use the original implementation provided by the authors; their codebase is not publicly released. For DiffFP, no implementation is available, so we reimplement the method from the paper, using diffusion policies within the fictitious-play update and its $Q$-guided action refinement procedure.

We additionally introduce Neural OMD as a controlled baseline to separate the effect of the OMD update from the expressive diffusion used by DiffNash. Neural OMD uses the same OMD training structure and optimization budget as DiffNash, but replaces the diffusion policy with a squashed diagonal-Gaussian policy parameterized by a two-layer MLP with hidden dimensions $(256,256)$ and \texttt{tanh} activations. Because this policy admits a tractable density, the mirror-descent regularizer is implemented using the analytic KL divergence between consecutive policies, with coefficient $1.0$. Neural OMD is trained on-policy using fresh samples at each outer iteration and doe not use a replay buffer for importance weighting.

For a given seed, all methods are evaluated on the same game instance and frozen anchor bank. The target distribution, kernel centering terms, and payoff parameters are therefore identical across methods.

\paragraph{Training details.}

All four experiments use five random seeds $0$-$4$. 

DiffNash and Neural OMD each use $2000$ outer iterations with $20$ gradient updates per iteration, for a total of $40,000$ updates per player. Both use Adam with a learning rate annealed from $3 \times 10^{-4}$ to $3\times 10^{-5}$, a policy batch size of $256$, and $128$ opponent samples. DiffNash parameterizes each policy with a $20$-step diffusion model with three hidden layers of width $256$ and a linear noise schedule scaled by 0.8. Its replay buffer if FIFO with capacity $1024$, with $32$ new samples added per outer iteration. Opponent samples are corrected using importance weights capped at $10$. For own-policy correction, $3$ of the $256$ samples in each batch are importance weighted, corresponding to approximately $1.17\%$ of the batch.

DiffFP uses a $20$-step diffusion model with hidden dimensions $(256,256)$ and a $32$-dimensional time embedding. We train it for $100$ fictitious-play iterations, with $400$ best-response updates per iteration, again giving $40{,}000$ gradient updates. It uses Adam \citep{kingma2014adam}  with learning rate $3\times10^{-4}$, gradient-norm clipping at $2.0$, training batches of $256$, and rollout batches of $32$. Its replay buffer has capacity $200{,}000$, with a pool of $100$ historical policies, and the best-response policy is reinitialized at each fictitious-play iteration. DiffFP does not use importance sampling.

DeepFP uses the authors' TensorFlow~1 implementation, with a two-layer $(128,128)$ policy network followed by a softmax output. We use Adam with learning rate $10^{-3}$ and train for $400{,}000$ games, applying a policy update every $10$ games.

\paragraph{Evaluation}

We measure convergence using the two-player zero-sum Nash gap \[ G(p,q) = \max_x U(x,q)-\min_y U(p,y), \] which is the sum of the two players' unilateral improvement incentives and is zero at equilibrium.

For symmetric Blotto, we follow the double-oracle evaluation of \citet{adam2021double}, computing best responses exactly on a discretized simplex with a knapsack dynamic program. Reference solutions use grid spacing $1/256$, while final Nash gaps use a finer $1/1024$ grid and $20{,}000$ policy samples per player. For the radial-band games, we approximate each best response using a pool of approximately $11{,}400$ candidates followed by $400$ gradient-refinement steps on the best $64$, and estimate payoffs using $8{,}000$ samples per player. We additionally report recovery of the known radial equilibrium, including near-extreme mass ($0.40$ at equilibrium) and active versus inactive radial mass. All metrics are aggregated over five seeds using the same evaluator for every method within each game.

\newpage

\newpage
% \input{appdx/cournot}
% \input{appdx/forest}
% \input{appdx/blotto}
% \input{appdx/prop1}
% % \input{appdx/lemma}
% \input{appdx/prop2}

\section{Missing Proofs}
\subsection{Proof of Proposition~\ref{lem:new-loss}}
\label{app:proof-rsm}

% \gwfpomdloss*
\setcounter{proposition}{0}
\begin{proposition}[GWFP and OMD Loss]
By choosing appropriate $g_{\OMD}(\vaction_t)$ and $g_{\GWFP}(\vaction_t)$
and substituting them into \Cref{eq:reweighted-formula}, we can equivalently
write their loss functions as
\begin{equation}\label{eq:appdx-OMD-loss}
\mathcal{L}_{\OMD}(\theta;t)
=
\mathbb{E}_{\vaction_0\sim\pi^{i,k}}
\left[
\mathbb{E}_{\vaction_t\sim q_{t|0}(\cdot\mid \vaction_0)}
\left[
\exp\left(\frac{1}{\eta}u^i(\vaction_0;\pi^{-i,k})\right)
\left\|
s_{\theta}(\vaction_t,t)
-
\nabla_{\vaction_t}\log q_{t|0}(\vaction_t\mid \vaction_0)
\right\|_2^2
\right].
\right]
\end{equation}
% _{G_{\theta^i}(\vaction_0,t)}
% \right].

\begin{equation}\label{eq:GWFP-loss}
\text{and} \quad
\mathcal{L}_{\mathrm{GWFP}}(\theta;t)
=
\mathbb{E}_{
    \substack{
    \vaction_t\sim h_t,\\
    \tilde{\vaction}_0\sim \phi_{0|t}
    }
}
\left[
    \exp\left(\frac{1}{\eta^k}
    u^i(\tilde{\vaction}_0;\Pi^{-i,k})\right)
    \left\|
        s_{\theta}(\vaction_t,t)
        -
        \nabla_{\vaction_t}
        \log \phi_{0|t}(\tilde{\vaction}_0\mid\vaction_t)
    \right\|_2^2
\right],
\end{equation}
where $h_t$ is a density which is positive almost anywhere and $\phi_{0\mid t}(\tilde{\vaction}_0 \mid \vaction_t)
\coloneqq
\mathcal{N}\!\left(
    \tilde{\vaction}_0;
    \frac{\vaction_t}{\sqrt{\bar{\alpha}_t}},
    \frac{1-\bar{\alpha}_t}{\bar{\alpha}_t}I
\right).$
For every $t \in (0,T]$, the minimizers of Equations~\eqref{eq:appdx-OMD-loss}
and~\eqref{eq:GWFP-loss} over $s$ are the score functions of
$\pi^{i,k+1}_{t,\OMD}(a_t)$ and
$\pi^{i,k+1}_{t,\mathrm{GWFP}}(a_t)$, respectively.
\end{proposition}

\begin{proof}
Fix player $i$, equilibrium-solving iteration $k$, and diffusion time $t\in(0,T]$. Recall that
\[
q_{t|0}(\vaction_t\mid \vaction_0)
=
\mathcal{N}\left(
    \vaction_t;
    \sqrt{\bar{\alpha}_t}\vaction_0,
    (1-\bar{\alpha}_t)I
\right).
\]
For either $\mathrm{solver}\in\{\mathrm{GWFP},\OMD\}$, define
\[
\pi^{i,k+1}_{t,\mathrm{solver}}(\vaction_t)
=
\int
\pi^{i,k+1}_{\mathrm{solver}}(\vaction_0)
q_{t|0}(\vaction_t\mid \vaction_0)\,d\vaction_0.
\]

We first recall the denoising score-matching identity. For any target density $\rho_0$ and its noised marginal
\[
\rho_t(\vaction_t)
=
\int
\rho_0(\vaction_0)q_{t|0}(\vaction_t\mid \vaction_0)\,d\vaction_0,
\]

differentiating under the integral gives

\begin{align}
\nabla_{\vaction_t}\log \rho_t(\vaction_t)
&=
\frac{
    \int
    \rho_0(\vaction_0)
    q_{t|0}(\vaction_t\mid \vaction_0)
    \nabla_{\vaction_t}\log q_{t|0}(\vaction_t\mid \vaction_0)
    \,d\vaction_0
}{
    \rho_t(\vaction_t)
}
\nonumber\\
&=
\mathbb{E}_{\vaction_0\sim \rho_{0|t}(\cdot\mid \vaction_t)}
\left[
    \nabla_{\vaction_t}\log q_{t|0}(\vaction_t\mid \vaction_0)
\right].
\label{eq:proof-tweedie}
\end{align}

For any nonnegative function $w_t(\vaction_t)$ for which the expectations are finite,
\begin{align}
&\int
w_t(\vaction_t)\rho_t(\vaction_t)
\left\|
    s_{\theta}(\vaction_t,t)
    -
    \nabla_{\vaction_t}\log \rho_t(\vaction_t)
\right\|_2^2
\,d\vaction_t
\nonumber\\
&\qquad =
\iint
w_t(\vaction_t)\rho_0(\vaction_0)q_{t|0}(\vaction_t\mid \vaction_0)
\left\|
    s_{\theta}(\vaction_t,t)
    -
    \nabla_{\vaction_t}\log q_{t|0}(\vaction_t\mid \vaction_0)
\right\|_2^2
\,d\vaction_0\,d\vaction_t
+
C_t,
\label{eq:weighted-dsm}
\end{align}

where $C_t$ does not depend on $\theta$. This follows directly from Equation~\eqref{eq:proof-tweedie} by expanding the squared norms.

\paragraph{OMD.}
From Equation~\eqref{eq:mirror-decent-update}, the OMD update is

\begin{equation}
\pi^{i,k+1}_{\OMD}(\vaction_0)
=
\frac{1}{Z^{i,k}_{\OMD}}
\pi^{i,k}(\vaction_0)
\exp\left(
    \frac{1}{\eta}
    u^i(\vaction_0;\pi^{-i,k})
\right).
\label{eq:proof-omd-target}
\end{equation}

In Equation~\eqref{eq:reweighted-formula}, set $p_t = \pi^{i,k+1}_{t, \OMD}$ and choose 
\begin{equation}
g_{\OMD}(\vaction_t)
=
Z^{i,k}_{\OMD}
\pi^{i,k+1}_{t,\OMD}(\vaction_t).
\label{eq:proof-g-omd}
\end{equation}

Substituting both $p_t$ and $g_\OMD$ into Equation~\eqref{eq:reweighted-formula} gives
\[
\mathcal{L}^{g_{\OMD}}(\theta;t)
=
Z_{\OMD}^{i,k}
\int
\pi_{t,\OMD}^{i,k+1}(\mathbf{a}_t)
\left\|
s_{\theta}(\mathbf{a}_t,t)
-
\nabla_{\mathbf{a}_t}
\log \pi_{t,\OMD}^{i,k+1}(\mathbf{a}_t)
\right\|_2^2
\, d\mathbf{a}_t.
\]
Then apply Equation~\ref{eq:weighted-dsm} with
\[
\rho_0 = \pi_{\OMD}^{i,k+1},
\qquad
\rho_t = \pi_{t,\OMD}^{i,k+1},
\qquad
w_t(\mathbf{a}_t) = Z_{\OMD}^{i,k}.
\]
Then
\[
\mathcal{L}^{g_{\OMD}}(\theta;t)
=
Z_{\OMD}^{i,k}
\iint
\pi_{\OMD}^{i,k+1}(\mathbf{a}_0)
q_{t|0}(\mathbf{a}_t \mid \mathbf{a}_0)
\left\|
s_{\theta}(\mathbf{a}_t,t)
-
\nabla_{\mathbf{a}_t}
\log q_{t|0}(\mathbf{a}_t \mid \mathbf{a}_0)
\right\|_2^2
\, d\mathbf{a}_0 \, d\mathbf{a}_t
+
C_t^{\OMD}, 
\]
where $C_t^{\OMD}$ does not depend on $\theta$. 
Substituting the OMD update and cancelling $Z_{\mathrm{OMD}}^{i,k}$ yields
\[
\begin{aligned}
\mathcal{L}^{g_{\mathrm{OMD}}}(\theta;t)
&=
\iint
\pi^{i,k}(\mathbf{a}_0)
q_{t|0}(\mathbf{a}_t \mid \mathbf{a}_0)
\exp\left(
\frac{1}{\eta}
u^i(\mathbf{a}_0;\pi^{-i,k})
\right)
\\
&\qquad \times
\left\|
s_{\theta}(\mathbf{a}_t,t)
-
\nabla_{\mathbf{a}_t}
\log q_{t|0}(\mathbf{a}_t \mid \mathbf{a}_0)
\right\|_2^2
\, d\mathbf{a}_0 \, d\mathbf{a}_t
+
C_t^{\mathrm{OMD}}
\\
&=
\mathbb{E}_{\mathbf{a}_0 \sim \pi^{i,k}}
\left[
\mathbb{E}_{\mathbf{a}_t \sim q_{t|0}(\cdot \mid \mathbf{a}_0)}
\left[
\exp\left(
\frac{1}{\eta}
u^i(\mathbf{a}_0;\pi^{-i,k})
\right)
\left\|
s_{\theta}(\mathbf{a}_t,t)
-
\nabla_{\mathbf{a}_t}
\log q_{t|0}(\mathbf{a}_t \mid \mathbf{a}_0)
\right\|_2^2
\right]
\right]
+
C_t^{\mathrm{OMD}}, 
\end{aligned}
\]
which is Equation~\eqref{eq:appdx-OMD-loss} up to an additive constant independent
of $\theta$.

% \newpage 

% Choose

% \begin{equation}
% g_{\OMD}(\vaction_t)
% =
% Z^{i,k}_{\OMD}
% \pi^{i,k+1}_{t,\OMD}(\vaction_t).
% \label{eq:proof-g-omd}
% \end{equation}

% Substituting Equation~\eqref{eq:proof-g-omd} into
% Equation~\eqref{eq:reweighted-formula} and applying
% Equation~\eqref{eq:weighted-dsm} gives

% \begin{align}
% \mathcal{L}^{g_{\OMD}}(\theta^i;t)
% &=
% \iint
% \pi^{i,k}(\vaction_0)
% q_{t|0}(\vaction_t\mid \vaction_0)
% \exp\left(
%     \frac{1}{\eta}
%     u^i(\vaction_0;\pi^{-i,k})
% \right)
% \nonumber\\[-1mm]
% &\qquad\qquad\quad\times
% \left\|
%     s_{\theta^i}(\vaction_t,t)
%     -
%     \nabla_{\vaction_t}\log q_{t|0}(\vaction_t\mid \vaction_0)
% \right\|_2^2
% \,d\vaction_t\,d\vaction_0
% +
% C_t^{\OMD}
% \nonumber\\
% &=
% \mathbb{E}_{\vaction_0\sim\pi^{i,k}}
% \left[
%     \mathbb{E}_{\vaction_t\sim q_{t|0}(\cdot\mid \vaction_0)}
%     \left[
%         \exp\left(
%             \frac{1}{\eta}
%             u^i(\vaction_0;\pi^{-i,k})
%         \right)
% \right.\right.
% \nonumber\\[-1mm]
% &\hspace{31mm}\left.\left.
%         \times
%         \left\|
%             s_{\theta^i}(\vaction_t,t)
%             -
%             \nabla_{\vaction_t}
%             \log q_{t|0}(\vaction_t\mid \vaction_0)
%         \right\|_2^2
%     \right]
% \right]
% +
% C_t^{\OMD}.
% \end{align}

% which is Equation~\eqref{eq:OMD-loss} up to an additive constant independent
% of $\theta^i$.

\paragraph{GWFP.}
From Equation~\eqref{eq:FP-update-rule}, the GWFP target is
\begin{equation}
\pi^{i,k+1}_{\mathrm{GWFP}}(\vaction_0)
=
\frac{1}{Z^{i,k}_{\mathrm{GWFP}}}
\exp\left(
    \frac{1}{\eta^k}
    u^i(\vaction_0;\Pi^{-i,k})
\right),
\label{eq:proof-gwfp-target}
\end{equation}

where $Z^{i,k}_{\mathrm{GWFP}}$ is the corresponding normalization constant. In Equation~\eqref{eq:reweighted-formula}, set $p_t = \pi^{i,k+1}_{t, \GWFP}$ and choose 
\begin{equation}
g_{\mathrm{GWFP}}(\vaction_t)
=
\bar{\alpha}_t^{d_i/2}
Z^{i,k}_{\mathrm{GWFP}}
h_t(\vaction_t)
\pi^{i,k+1}_{t,\mathrm{GWFP}}(\vaction_t),
\label{eq:proof-g-gwfp}
\end{equation}
where $d^i$ is player $i$'s action dimension. Substituting both $p_t$ and $g_\GWFP$ into Equation~\eqref{eq:reweighted-formula} and applying Equation~\eqref{eq:weighted-dsm} gives
\begin{align}
\mathcal{L}^{g_{\mathrm{GWFP}}}(\theta;t)
&=
\bar{\alpha}_t^{d_i/2}
\iint
h_t(\vaction_t)
\exp\left(
    \frac{1}{\eta^k}
    u^i(\vaction_0;\Pi^{-i,k})
\right)
q_{t|0}(\vaction_t\mid \vaction_0)
\nonumber\\[-1mm]
&\qquad\qquad\quad\times
\left\|
    s_{\theta}(\vaction_t,t)
    -
    \nabla_{\vaction_t}\log q_{t|0}(\vaction_t\mid \vaction_0)
\right\|_2^2
\,d\vaction_0\,d\vaction_t
+
C_t^{\mathrm{GWFP}}.
\label{eq:proof-gwfp-mid}
\end{align}

Now define
\[
\phi_{0|t}(\vaction_0\mid \vaction_t)
=
\mathcal{N}\left(
    \vaction_0;
    \frac{\vaction_t}{\sqrt{\bar{\alpha}_t}},
    \frac{1-\bar{\alpha}_t}{\bar{\alpha}_t}I
\right).
\]

Comparing the Gaussian densities gives
\begin{equation}
\phi_{0|t}(\vaction_0\mid \vaction_t)
=
\bar{\alpha}_t^{d_i/2}
q_{t|0}(\vaction_t\mid \vaction_0),
\label{eq:proof-reverse-density}
\end{equation}
and therefore
\begin{equation}
\nabla_{\vaction_t}\log \phi_{0|t}(\vaction_0\mid \vaction_t)
=
\nabla_{\vaction_t}\log q_{t|0}(\vaction_t\mid \vaction_0)
=
-\frac{
    \vaction_t-\sqrt{\bar{\alpha}_t}\vaction_0
}{
    1-\bar{\alpha}_t
}.
\label{eq:proof-reverse-score}
\end{equation}

Applying Equations \eqref{eq:proof-reverse-density} and \eqref{eq:proof-reverse-score} to Equation \eqref{eq:proof-gwfp-mid} yields

\begin{align}
\mathcal{L}^{g_{\mathrm{GWFP}}}(\theta;t)
&=
\mathbb{E}_{\substack{
\vaction_t\sim h_t\\
\tilde \vaction_0\sim\phi_{0|t}(\cdot\mid \vaction_t)
}}
\left[
\exp\left(
    \frac{1}{\eta^k}
    u^i(\tilde \vaction_0;\Pi^{-i,k})
\right)
\right.
\nonumber\\[-1mm]
&\hspace{27mm}\left.
\times
\left\|
    s_{\theta}(\vaction_t,t)
    -
    \nabla_{\vaction_t}
    \log\phi_{0|t}(\tilde \vaction_0\mid \vaction_t)
\right\|_2^2
\right]
+
C_t^{\mathrm{GWFP}},
\end{align}

which is the GWFP loss up to an additive constant independent of $\theta$.

\end{proof}

Finally, Gaussian corruption implies $\pi^{i,k+1}_{t,\mathrm{GWFP}}(\vaction_t)>0$ and $\pi^{i,k+1}_{t,\OMD}(\vaction_t)>0$ for $t>0$. Because $h_t$ is strictly positive on its support, both $g_{\mathrm{GWFP}}$ and $g_{\OMD}$ are positive almost everywhere. Equation~\eqref{eq:reweighted-formula} is therefore globally minimized precisely when
\[
s_{\theta}(\vaction_t,t)
=
\nabla_{\vaction_t}
\log \pi^{i,k+1}_{t,\mathrm{solver}}(\vaction_t)
\]
almost everywhere. This proves the result for both
$\mathrm{solver}\in\{\mathrm{GWFP},\OMD\}$.

\subsection{Proof of \Cref{prop:elbo-difference}}\label{appdx:noise-prediction-form}
\noiseprediction*

\begin{proof}
Fix $\vaction_0$. By the ELBO decomposition in \citet{ho2020denoising},
\begin{align}
\operatorname{ELBO}_{\theta}(\vaction_0)
={}&
\underbrace{
\mathbb{E}_{q(\vaction_1\mid \vaction_0)}
\bigl[\log p_\theta(\vaction_0\mid \vaction_1)\bigr]
}_{\text{Reconstruction term}}
\nonumber\\
&-
\underbrace{
\sum_{t=2}^{T}
\mathbb{E}_{q(\vaction_t\mid \vaction_0)}
\left[
D_{\mathrm{KL}}
\left(
q(\vaction_{t-1}\mid \vaction_t,\vaction_0)
\,\|\,p_\theta(\vaction_{t-1}\mid \vaction_t)
\right)
\right]
}_{\text{Denoising KL terms}}
\nonumber\\
&-
\underbrace{
D_{\mathrm{KL}}\!\left(
q(\vaction_T\mid \vaction_0)
\,\|\,p(\vaction_T)
\right)
}_{\text{Terminal prior KL term}}.
\label{eq:elbo-decomp-prop}
\end{align}

% \begin{align}
% \operatorname{ELBO}_{\theta}(\vaction_0)
% ={}&
% \mathbb{E}_{q(\vaction_1\mid \vaction_0)}
%     \bigl[\log p_\theta(\vaction_0\mid \vaction_1)\bigr]
% \nonumber\\
% &-
% \sum_{t=2}^{T}
% \mathbb{E}_{q(\vaction_t\mid \vaction_0)}
% \left[
% D_{\mathrm{KL}}
% \left(
% q(\vaction_{t-1}\mid \vaction_t,\vaction_0)
% \,\|\,p_\theta(\vaction_{t-1}\mid \vaction_t)
% \right)
% \right]
% \nonumber\\
% &-
% \underbrace{
% D_{\mathrm{KL}}\!\left(q(a_T\mid a_0)\,\|\,p(a_T)\right)
% }_{\text{Terminal prior}}.
% \label{eq:elbo-decomp-prop}
% \end{align}

% \begin{align}
% \operatorname{ELBO}_{\theta}(\vaction_0)
% ={}&
% \mathbb{E}_{q(\vaction_1\mid \vaction_0)}
%     \bigl[\log p_\theta(\vaction_0\mid \vaction_1)\bigr]
% \nonumber\\
% &-
% \sum_{t=2}^{T}
% \mathbb{E}_{q(\vaction_T\mid \vaction_0)}
% \left[
% D_{\mathrm{KL}}
% \left(
% q(\vaction_{t-1}\mid \vaction_T,\vaction_0)
% \,\|\,p_\theta(\vaction_{t-1}\mid \vaction_T)
% \right)
% \right]
% \nonumber\\
% &-
% D_{\mathrm{KL}}
% \left(q(\vaction_T\mid \vaction_0)\,\|\,p(\vaction_T)\right).
% \label{eq:elbo-decomp-prop}
% \end{align}

Under the forward process,
\[
\vaction_t
=
\sqrt{\bar{\alpha}_t}\,\vaction_0
+
\sqrt{1-\bar{\alpha}_t}\,\epsilon,
\qquad
\epsilon\sim\mathcal{N}(0,I).
\]
This implies
\begin{equation}\label{eq:noise-equivalent-form}
    \epsilon
=
\frac{\mathbf{a}_t-\sqrt{\bar{\alpha}_t}\,\mathbf{a}_0}
{\sqrt{1-\bar{\alpha}_t}}.
\end{equation}
\paragraph{Proof Strategy} 
We first express the reconstruction term and each denoising KL term as an expectation of $D(\theta;a_0,t,\epsilon)$, with additive constants independent of $\theta$. We then subtract the two ELBOs: the shared terminal prior KL term and additive constants cancel, leaving a difference of expected $D$ terms from which we construct an unbiased estimator.
\paragraph{Denoising KL Terms}
Our goal here is to prove that $D_{\mathrm{KL}} \bigg(q(\vaction_{t-1}\mid \vaction_t,\vaction_0)
\,\|\,p_\theta(\vaction_{t-1}\mid \vaction_t)\bigg)$ is equivalent to $D(\theta;a_0, t, \epsilon)$ up to a constant that does not depend on $\theta$. 

For $t \geq 2$, citing Equation 6 in \citet{ho2020denoising} we know that:
\begin{equation}
q(\vaction_{t-1}\mid \vaction_t,\vaction_0)
=
\mathcal{N}\!\left(
\vaction_{t-1};
\tilde{\mu}_t(\vaction_t,\vaction_0),
\tilde{\beta}_t I
\right),
\end{equation}
where $\tilde{\beta}_t
=
\frac{1-\bar{\alpha}_{t-1}}{1-\bar{\alpha}_t}\beta_t$ and $\tilde{\mu}_t(\vaction_t,\vaction_0)
=
\frac{\sqrt{\bar{\alpha}_{t-1}}\beta_t}
     {1-\bar{\alpha}_t}\vaction_0
+
\frac{\sqrt{\alpha_t}(1-\bar{\alpha}_{t-1})}
     {1-\bar{\alpha}_t}\vaction_t.$
     
Use \Cref{eq:noise-equivalent-form} and 
$\bar{\alpha}_t=\alpha_t\bar{\alpha}_{t-1}$, we can rewrite this mean $\tilde{\mu}_t(\vaction_t,\vaction_0)$ in terms of $\epsilon$:
\begin{align}
&\frac{1}{\sqrt{\alpha_t}}
\left(
\vaction_t
-
\frac{\beta_t}{\sqrt{1-\bar{\alpha}_t}}\epsilon
\right)
\nonumber\\
&\quad=
\frac{1}{\sqrt{\alpha_t}}
\left(
\vaction_t
-
\frac{\beta_t}{1-\bar{\alpha}_t}
\left(
\vaction_t-\sqrt{\bar{\alpha}_t}\vaction_0
\right)
\right)
\nonumber\\
&\quad=
\frac{1-\bar{\alpha}_t-\beta_t}
     {\sqrt{\alpha_t}(1-\bar{\alpha}_t)}\vaction_t
+
\frac{\beta_t\sqrt{\bar{\alpha}_t}}
     {\sqrt{\alpha_t}(1-\bar{\alpha}_t)}\vaction_0
\nonumber\\
&\quad=
\frac{\sqrt{\alpha_t}(1-\bar{\alpha}_{t-1})}
     {1-\bar{\alpha}_t}\vaction_t
+
\frac{\beta_t\sqrt{\bar{\alpha}_{t-1}}}
     {1-\bar{\alpha}_t}\vaction_0
\nonumber\\
&\quad=
\tilde{\mu}_t(\vaction_t,\vaction_0).
\end{align}
To obtain the third equality above, we use the identity $1-\bar{\alpha}_t-\beta_t
=
\alpha_t(1-\bar{\alpha}_{t-1})$.

Now we focus on the the reverse transition which has fixed covariance $p_\theta(\vaction_{t-1}\mid \vaction_t)
=
\mathcal{N}\!\left(
\vaction_{t-1};
\mu_\theta(\vaction_t,t),
\sigma_t^2 I
\right)$ with score-parameterized mean $\mu_\theta(\vaction_t,t)
=
\frac{1}{\sqrt{\alpha_t}}
\left(
\vaction_t+\beta_t s_\theta(\vaction_t,t)
\right).$
Subtracting the two means gives
\begin{align}
\tilde{\mu}_t(\vaction_t,\vaction_0)
-
\mu_\theta(\vaction_t,t)
&=
\frac{1}{\sqrt{\alpha_t}}
\left[
\vaction_t
-
\frac{\beta_t}{\sqrt{1-\bar{\alpha}_t}}\epsilon
-
\vaction_t
-
\beta_t s_\theta(\vaction_t,t)
\right]
\nonumber\\
&=
-\frac{\beta_t}{\sqrt{\alpha_t}}
\left(
s_\theta(\vaction_t,t)
+
\frac{\epsilon}{\sqrt{1-\bar{\alpha}_t}}
\right).
\end{align}
Consequently,
\begin{equation}
\left\|\tilde{\mu}_t-\mu_\theta\right\|_2^2
=
\frac{\beta_t^2}{\alpha_t}
\left\|
s_\theta(\vaction_t,t)
+
\frac{\epsilon}{\sqrt{1-\bar{\alpha}_t}}
\right\|_2^2.
\label{eq:posterior-reverse-mean-difference}
\end{equation}
Recall $d$ is the dimension of $\vaction_t$. Now we can focus on the denoising KL terms. Note the KL between these two Gaussians is 
\begin{align*}
    D_{\mathrm{KL}}\!\left(
q(\vaction_{t-1}\mid \vaction_t,\vaction_0)
\,\|\,
p_\theta(\vaction_{t-1}\mid \vaction_t)
\right)
\nonumber =
\frac{1}{2}
\left[
d\frac{\tilde{\beta}_t}{\sigma_t^2}
-d
+d\log\frac{\sigma_t^2}{\tilde{\beta}_t}
+
\frac{\left\|\tilde{\mu}_t-\mu_\theta\right\|_2^2}
     {\sigma_t^2}
\right].
\end{align*}
Plug in the expression of mean difference in \Cref{eq:posterior-reverse-mean-difference} we obtain
\begin{align}
D_{\mathrm{KL}}(q\|p_\theta)
=
\frac{d}{2}
\left(
\frac{\tilde{\beta}_t}{\sigma_t^2}
-1
+\log\frac{\sigma_t^2}{\tilde{\beta}_t}
\right)
\nonumber +
\underbrace{
\frac{\beta_t^2}{2\sigma_t^2\alpha_t}
}_{C_t}
\left\|
s_\theta(\vaction_t,t)
+
\frac{\epsilon}{\sqrt{1-\bar{\alpha}_t}}
\right\|_2^2
\nonumber
=
c_t+D(\theta;\vaction_0,t,\epsilon).
\end{align}
We denote $C_t = \frac{\beta_t^2}{2\sigma_t^2\alpha_t}$, a constant that does not depend on $\theta$. This finishes our goal that the denoising KL term is equivalent to $D(\theta; a_0, t, \epsilon)$. 

\paragraph{Reconstruction Term}
We next consider the reconstruction term at $t=1$. Since $\bar{\alpha}_1=\alpha_1$ and
$1-\bar{\alpha}_1=\beta_1$, the forward sample satisfies $\vaction_1
=
\sqrt{\alpha_1}\vaction_0+\sqrt{\beta_1}\epsilon$ Therefore, we can rewrite $\vaction_0
=
\frac{\vaction_1-\sqrt{\beta_1}\epsilon}
{\sqrt{\alpha_1}}.$ Subtracting the reverse-model mean $\mu_\theta(\vaction_1, 1)$ gives
\begin{align}\label{eq:reconstruction-mean-difference}
\vaction_0-\mu_\theta(\vaction_1,1)
&=
\frac{
\vaction_1-\sqrt{\beta_1}\epsilon
-\vaction_1-\beta_1s_\theta(\vaction_1,1)
}{\sqrt{\alpha_1}}
\nonumber\\
&=
-\frac{\beta_1}{\sqrt{\alpha_1}}
\left(
s_\theta(\vaction_1,1)
+\frac{\epsilon}{\sqrt{\beta_1}}
\right).
\end{align}

For the Gaussian decoder
$p_\theta(\vaction_0\mid\vaction_1)
=\mathcal N(\vaction_0;\mu_\theta(\vaction_1,1),\sigma_1^2I)$
with fixed $\sigma_1^2>0$, by plugging in \Cref{eq:reconstruction-mean-difference}, we obtain
\begin{align}
\log p_\theta(\vaction_0\mid\vaction_1)
&=
-\frac{d}{2}\log(2\pi\sigma_1^2)
-\frac{
\|\vaction_0-\mu_\theta(\vaction_1,1)\|_2^2
}{2\sigma_1^2}
\nonumber\\
&=
-\frac{d}{2}\log(2\pi\sigma_1^2)
-\underbrace{
\frac{\beta_1^2}{2\sigma_1^2\alpha_1}
}_{C_1}
\left\|
s_\theta(\vaction_1,1)
+\frac{\epsilon}{\sqrt{1-\bar{\alpha}_1}}
\right\|_2^2
\nonumber\\
&=
-\frac{d}{2}\log(2\pi\sigma_1^2)
-D(\theta;\vaction_0,1,\epsilon).
\end{align}

The coefficient $C_1$ has the same form as
$C_t=\beta_t^2/(2\sigma_t^2\alpha_t)$ for $t\ge2$. 

\paragraph{Combining the Reconstruction Term and the denoising KL term}
Because the two policy iterates share the same forward schedule,
reverse-variance schedule, and terminal prior, the same $\const(\vaction_0)$ applies
to both. Consequently, the constant terms get cancelled and we obtain
\begin{align}
\operatorname{ELBO}^{k}(\vaction_0)
-
\operatorname{ELBO}^{\behavioriteration}(\vaction_0)
&=
\sum_{t=1}^{T}
\mathbb E_{\epsilon}
\left[
D(\theta^{\behavioriteration};\vaction_0,t,\epsilon)
-
D(\theta^{k};\vaction_0,t,\epsilon)
\right].
\label{eq:elbo-difference-proof}
\end{align}

Finally, because $t\sim\operatorname{Unif}\{1,\ldots,T\}$,
\begin{align*}
\mathbb E_{t,\epsilon}[\isestimator]
&=
T\,
\mathbb E_{t,\epsilon}
\left[
D(\theta^{\behavioriteration};\vaction_0,t,\epsilon)
-
D(\theta^{k};\vaction_0,t,\epsilon)
\right] \\
&=
\sum_{t=1}^{T}
\mathbb E_{\epsilon}
\left[
D(\theta^{\behavioriteration};\vaction_0,t,\epsilon)
-
D(\theta^{k};\vaction_0,t,\epsilon)
\right] \\
&=
\operatorname{ELBO}^{\behavioriteration}(\vaction_0)
-
\operatorname{ELBO}^{k}(\vaction_0).
\end{align*}
Thus, $\isestimator$ is an unbiased estimator of the ELBO difference.
\end{proof}

\section{Additional Experimental Results}
\label{app:additional-results}

\subsection{Additional Forest Game Visualizations}
\label{forest-vis}

Figure \ref{fig:forest-all} provides the complete strategy visualization for all six forest configurations reported in Table \ref{tab:forest-results}. Although the payoff rule is unchanged, each tree layout induces a different spatial equilibrium. 

\begin{figure*}[h]
    \centering
    \includegraphics[width=\textwidth]
        {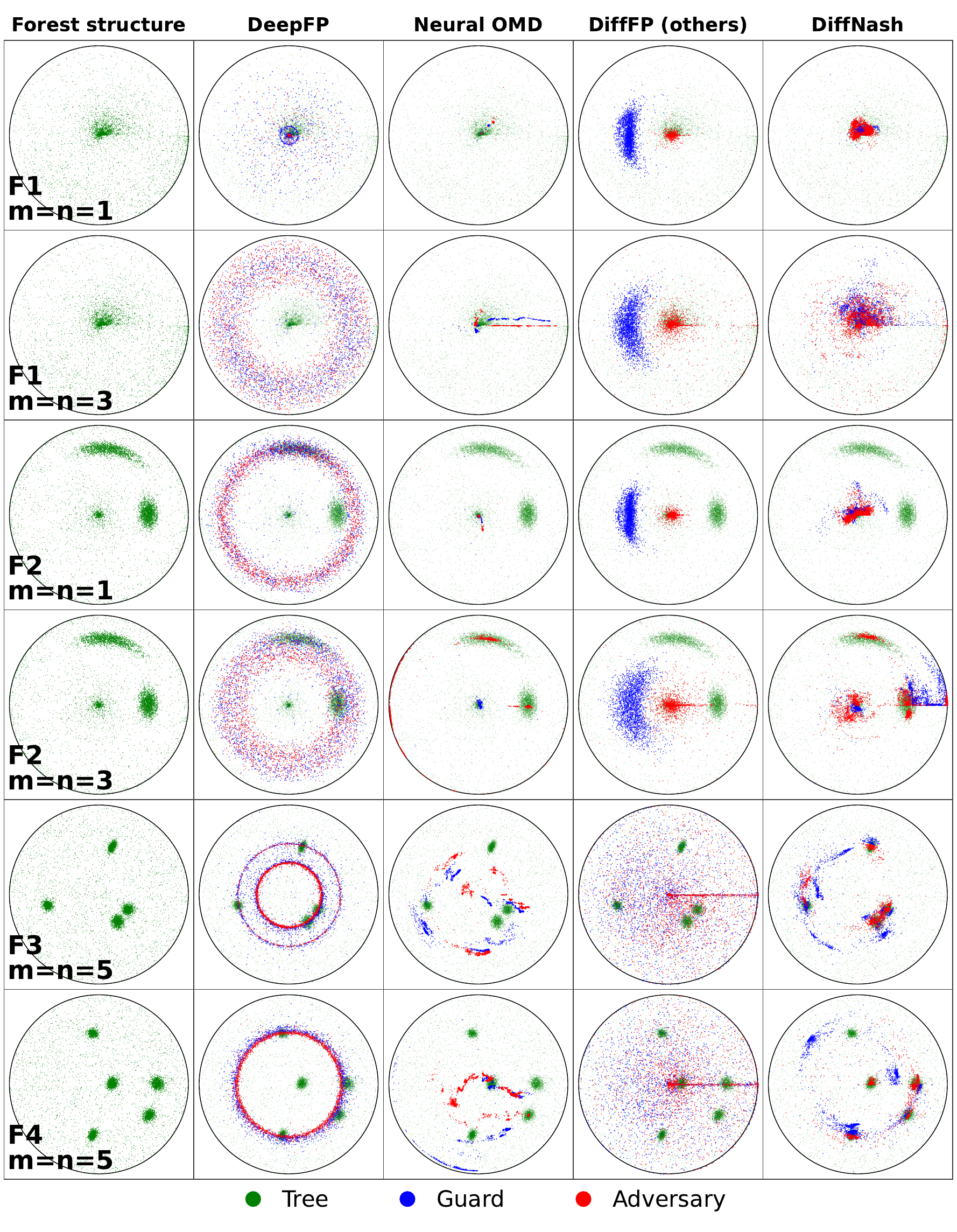}
    \caption{
        \textbf{Learned strategies across all forest security games.}
        Green points denote trees, blue points defender actions, and red points
        attacker actions. Rows vary both the forest geometry and the number of
        resources $m=n$. DiffNash obtains the lowest mean Nash gap in all six
        configurations; Figure~\ref{fig:forest} highlights two representative
        multi-resource cases in the main text.
    }
    \label{fig:forest-all}
\end{figure*}

\subsection{10-dimensional Continuous Colonel Blotto}
\label{app:10d-blotto}

Figure \ref{fig:blotto-all} shows the learned strategy distributions for the $n=6$ and $n=10$ Blotto experiments reported in Table \ref{tab:blotto-results}. 

\begin{figure*}[t]
    \centering
    \includegraphics[width=\textwidth]
        {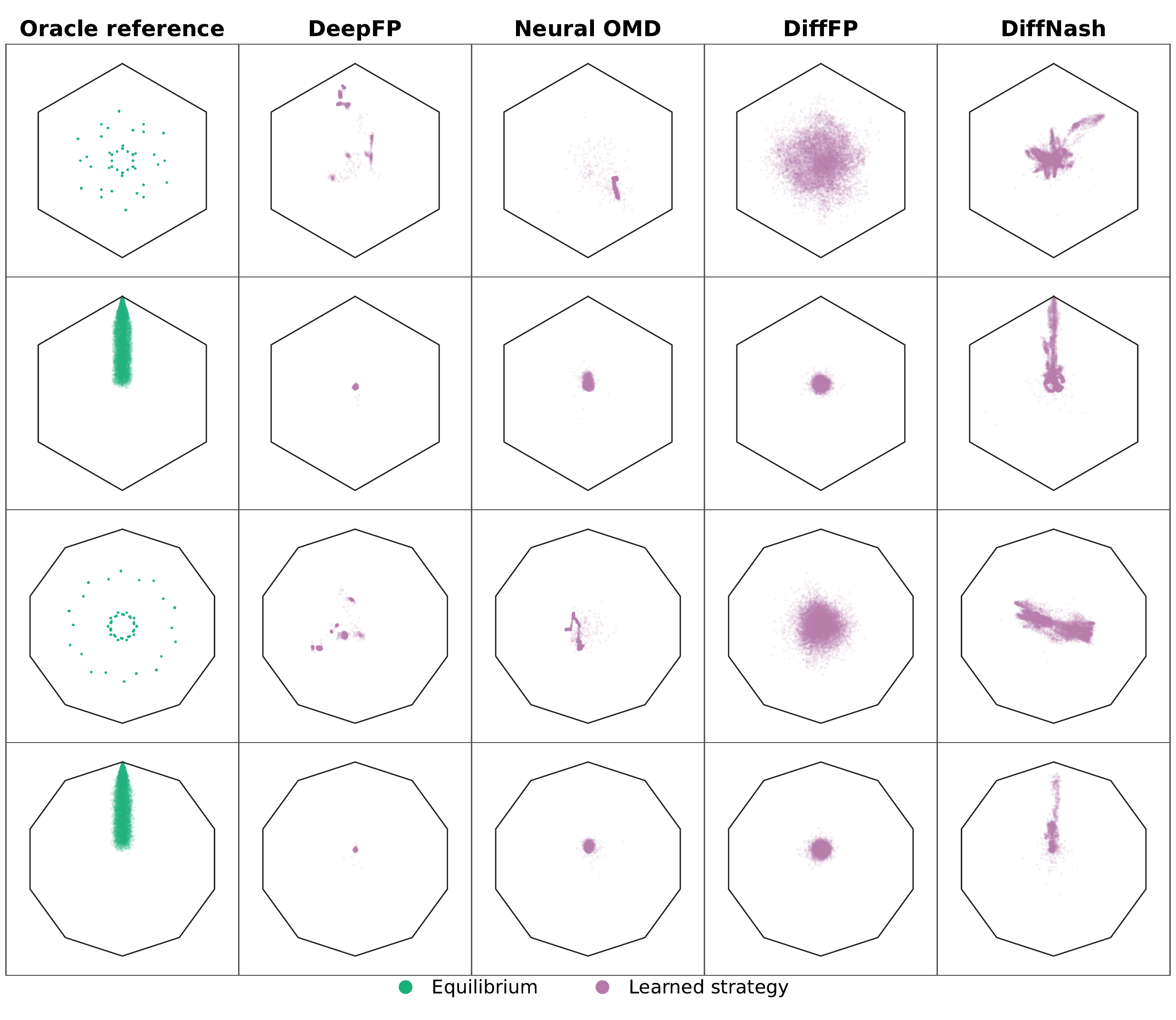}
    \caption{
        \textbf{Continuous Colonel Blotto strategy distributions.}
        Representative learned strategies projected into the polygonal action-space visualization. Columns show the oracle reference and each method; rows correspond to symmetric and radial-band Blotto for $n=6$ and $n=10$, respectively. Green denotes equilibrium samples and purple the learned strategy.
    }
    \label{fig:blotto-all}
\end{figure*}

\subsection{General-Sum Forest Protection}
\label{app:general-sum}

We also consider a general-sum variant of the F2 forest protection game found in \cite{kamra2019deep}. We retain the original F2 geometry and interaction mechanics, but assign each tree separate ecological and economic values. This creates different objectives for the defender and poacher without changing the continuous action space.

Let $v_t^D$ and $v_t^P$ denote the ecological and economic values of tree $t$, respectively. For each adversary logger $l$, let $a_l = 1$ indicate that it is captured by at least one guard and let $s_{t,l}$ denote its share of tree $t$ when multiple adversary target the same tree. The payoffs are

\begin{align}
    u^D
    &=
    \sum_l
    \left[
        -(1-a_l)\sum_t s_{t,l}v_t^D
        + 4a_l
    \right], \\
    u^P
    &=
    \sum_l
    \left[
        (1-a_l)\sum_t s_{t,l}v_t^P
        - 4a_l
    \right].
\end{align}

where $a_l$ indicates whether adversary $l$ is captured and $s_{(t,l}$ is its share of tree $t$. The original zero-sum game is recovered when $v_t^D =v_t^P = 1$.

We construct the two value fields to favor different parts of the forest. For a tree $t$, let $r_t$ denote its radial distance from the center and let $k_t$ be the number of other trees within Euclidean distance $0.05$. We define

\begin{equation}
    \widetilde v_t^D
    =
    \frac{\exp(-r_t/0.3)}{1+k_t},
    \qquad
    \widetilde v_t^P
    =
    \exp\left(-\frac{1-r_t}{0.3}\right)(1+k_t).
\end{equation}

and independently rescale and clip each field to $[0.25, 4]$ while preserving mean 1, so that the overall payoff scale is comparable to the original game. Thus, the defender places greater value on isolated interior trees, while the poacher focuses more on dense patches near the boundary. On F2, the two value fields have correlation $-0.52$. 

Figure \ref{fig:forest-general-sum} compares DeepFP, Neural OMD, DiffFP, and DiffNash for $m=n \in \{1,3\}$. DiffNash attains the lowest mean Nash gap, while DeepFP and DiffFP remain more exploitable. Neural OMD performs better than the FP-based methods, but exhibits higher variance than DiffNash. The visualized seed also shows qualitatively different spatial strategies across methods. 

\begin{figure*}[t]
    \centering
    \includegraphics[width=\textwidth]
        {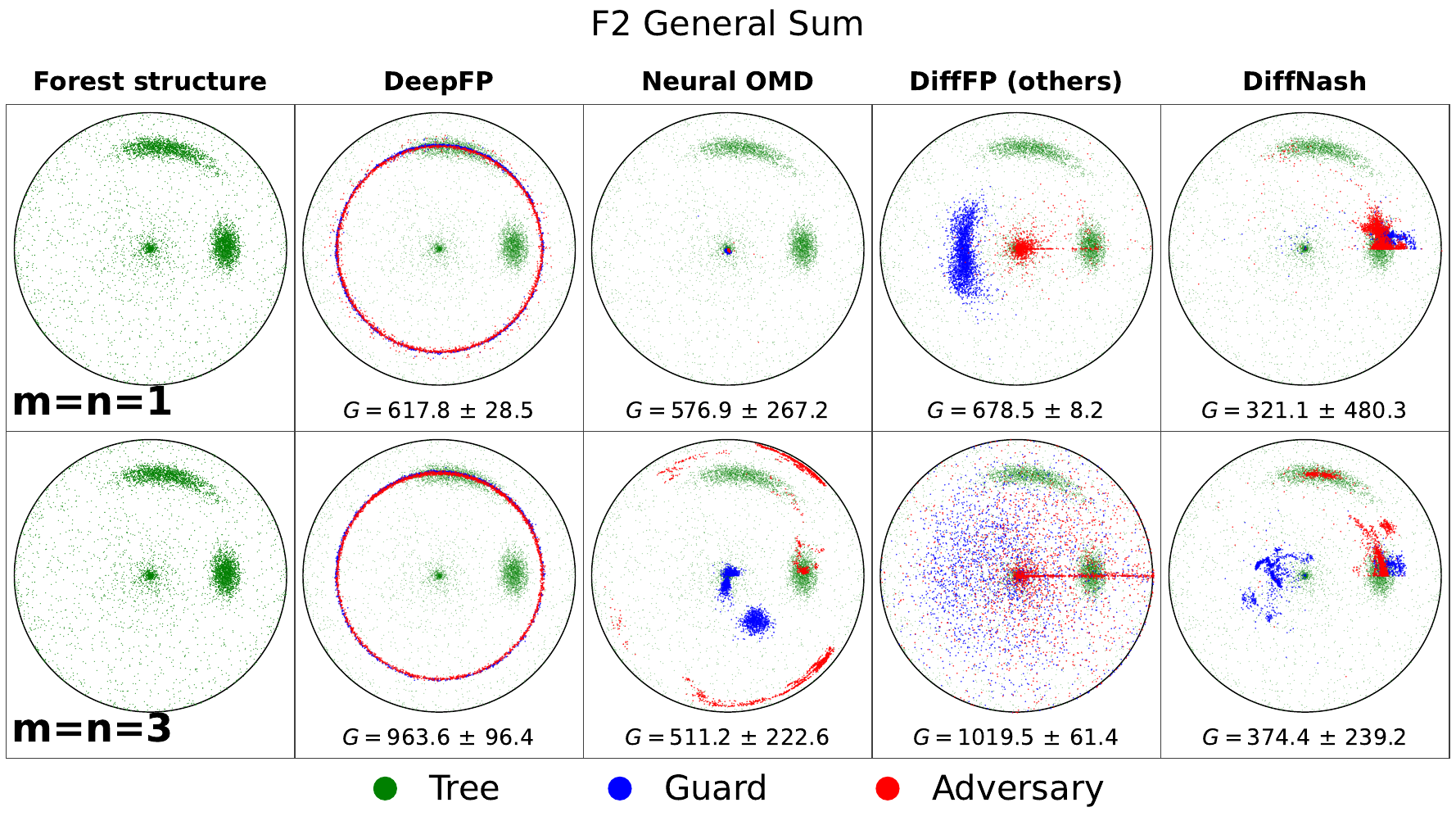}
    \caption{
        \textbf{General-sum Forest Game.}
        Learned strategies on the general-sum F2 forest game for \(m=n\in\{1,3\}\). Metrics are mean \(\pm\) 95\% confidence intervals over three seeds.
    }
    \label{fig:forest-general-sum}
\end{figure*}

\subsection{Importance-sampling Forest ablation}
\label{app:forest-ablation}

\begin{figure*}[t]
    \centering
    \includegraphics[width=0.7\textwidth]
{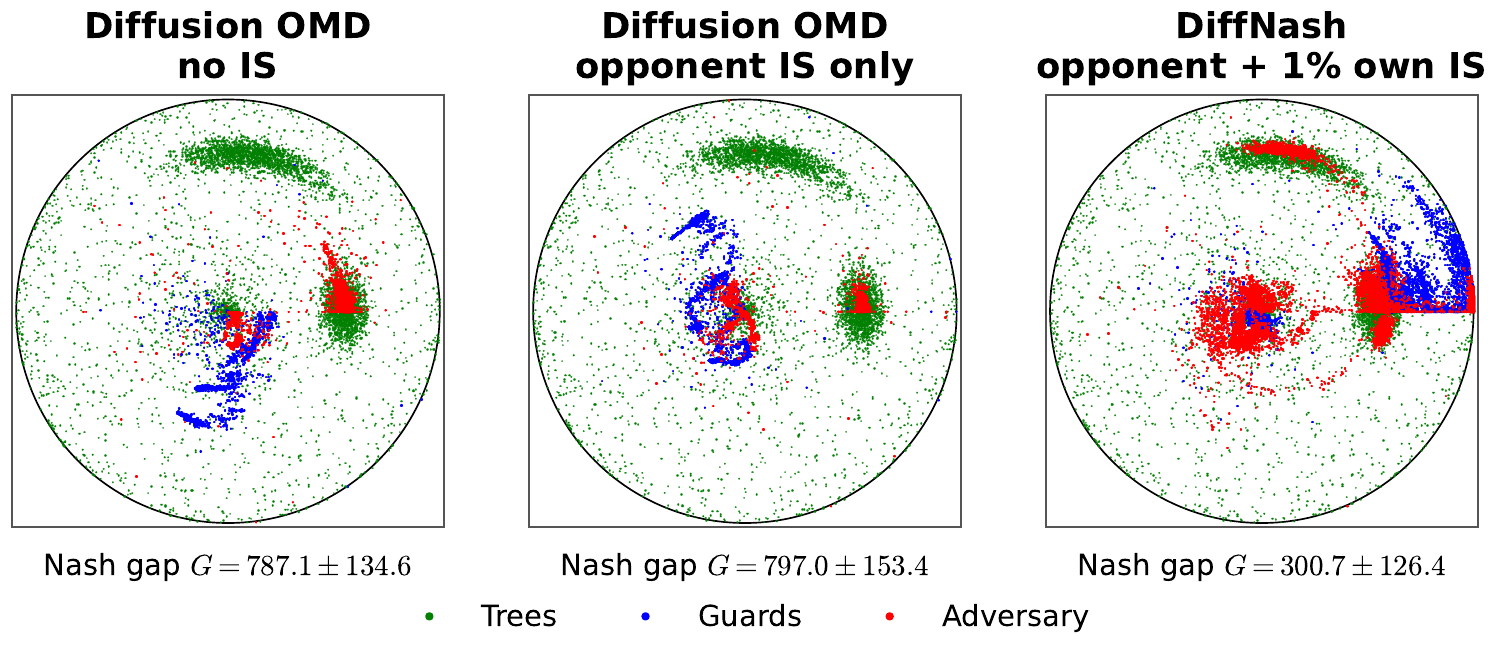}
    \caption{
        \textbf{Importance-sampling ablation on F2.}
    Removing own-policy importance sampling substantially degrades DiffNash, whereas opponent IS has little effect on convergence when own IS is omitted.
    }
    \label{fig:forest-is-ablation}
\end{figure*}

We use F2, where $m=n=3$, to isolate the roles of the two importance-sampling corrections in DiffNash. In particular, we separate the correction applied to replayed samples from the player's own past policies (\emph{own IS}) from the correction used when evaluating samples from past opponent policies (\emph{opponent IS}). Figure~\ref{fig:forest-is-ablation} shows that the two corrections have different effects. Removing own IS substantially degrades performance, increasing the Nash gap from $300.7\pm126.4$ for the full DiffNash method to $797.0\pm153.4$ with opponent IS alone. In contrast, opponent IS has little effect on convergence when own IS is omitted: the Nash gap is $787.1\pm134.6$ without either correction and $797.0\pm153.4$ with opponent IS alone. This is consistent with the Cournot ablation in the main paper (Figure \ref{fig:cournot-ablation}, where opponent IS likewise has little effect on convergence performance, but reduces score-network calls to roughly $66\%$ of those required without opponent IS. Together, these results suggest that own IS provides the primary optimization benefit by correcting the distribution shift induced by reusing the player's historical diffusion samples, while opponent IS mainly improves computational efficiency.

\subsection{Forest Game Hyperparameter Sweep}
\label{app:hyperparameter-sweep}
We fix other hyperparamters and change self action rate $\zeta$ in \Cref{fig:self-action-is-hyperparamter-study}. \Cref{fig:self-action-is-hyperparamter-study} shows that lower self-action IS rates generally yield smaller mean Nash gaps. Because importance weights are clipped, including too many samples from past iterations may bias the OMD loss in \Cref{eq:OMD-loss}. Empirically, a $1\%$ self-action IS rate suffices to stabilize training.

% , we find that lower self action IS rate generally lowers Nash gap. This is because the importance weight is clipped, so too many past iteration samples may bias the OMD training loss \Cref{eq:OMD-loss} too much. We empirically find that $1\%$ of self action IS sample is enough to stablize training. 

In Table \ref{tab:hyperparameter-sweep}, we study sensitivity to the own-action batch size, opponent batch size, and trigger threshold $\kappa$ on F2 and F3 over five seeds. 

\paragraph{Trigger $\kappa$}: The default $\kappa = 10\%$ achieves a slightly lower Nash gap on F2 than the other trigger values; otherwise, no trigger value consistently outperforms the others, demonstrating that \(\mathbf{DiffNash}\) is relatively robust to this hyperparameter. 

\paragraph{Own Buffer Size}: Increasing the own-action batch size to \(512\) degrades exploitability by \(25.9\%\) and the Nash gap by \(28.2\%\) relative to the default, suggesting that a moderate own-action batch size combined with self action importance sampling is more robust. In F2, the default batch size modestly outperforms a batch size of 128.

\paragraph{Opponent Buffer Size}: The default opponent batch size wins in all but one setting in terms of defender exploitability and Nash gap.

% For the batch-size and trigger sweeps, the default configuration performs best on average across both defender exploitability and Nash gap, although we observe no clear monotonic trend. The effect of $\zeta$ is more game-dependent: increasing it above the default $1\%$ worsens performance on F2, while F3 improves at $5\%$.

\begin{table}[t]
\caption{\textbf{Hyperparameter sensitivity on F2 (m=n=3) and F3 (m=n=5).}
Each row changes one parameter from the default DiffNash configuration
(own batch 256, opponent batch 128, $\zeta=1\%$, $\kappa=10\%$).
We report defender exploitability $\mathrm{Exploit_D}$ and Nash gap $G$;
lower is better. We add $95\%$ confidence interval to each value.}
\label{tab:hyperparameter-sweep}
\centering
\scriptsize
\setlength{\tabcolsep}{3pt}
\begin{tabular}{llcccc}
\toprule
& & \multicolumn{2}{c}{F2} & \multicolumn{2}{c}{F3} \\
Setting & Value &
$\mathrm{Exploit_D}$ & $\mathrm{Gap}$ &
$\mathrm{Exploit_D}$ & $\mathrm{Gap}$ \\
\midrule
\textbf{Default}
& --
& $\mathbf{322.5}$ $\pm$ 110.1 & $\mathbf{300.7}$ $\pm$ 126.4
& $559.7$ $\pm$ 76.4 & $\mathbf{528.2}$ $\pm$ 142.7 \\
\midrule
Own batch
& 128
& 388.9 $\pm$ 62.0 & 358.7 $\pm$ 88.2
& $\mathbf{553.5}$ $\pm$ 250.5 & 510.5 $\pm$ 271.8 \\
& 512
& 441.5 $\pm$ 53.0 & 431.5 $\pm$ 66.7
& 669.2 $\pm$ 87.8 & 631.4 $\pm$ 75.8 \\
\midrule
Opponent batch
& 64
& 442.0 $\pm$ 179.8 & 419.5 $\pm$ 200.5
& 679.8 $\pm$ 289.7 & 667.2 $\pm$ 292.1 \\
& 256
& 397.9 $\pm$ 129.4 & 385.5 $\pm$ 141.3
& 576.7 $\pm$ 209.6 & 485.0 $\pm$ 246.6 \\
\midrule
% IS fraction $\zeta$
% & $2.5\%$
% & 385.8 $\pm$ 131.1 & 372.0 $\pm$ 126.9
% & 591.4 $\pm$ 56.5 & 538.4 $\pm$ 68.3 \\
% & $5\%$
% & 502.5 $\pm$ 176.6 & 502.7 $\pm$ 187.4
% & 391.2 $\pm$ 363.5 & 348.6 $\pm$ 374.6 \\
% \midrule
Trigger $\kappa$
& $5\%$
& 372.2 $\pm$ 39.9 & 361.9 $\pm$ 51.6
& 576.7 $\pm$ 94.7 & 531.8 $\pm$ 126.4 \\
& $15\%$
& 373.0 $\pm$ 93.5 & 366.2 $\pm$ 81.2
& 570.3 $\pm$ 189.1 & \textbf{502.0} $\pm$ 245.7 \\
\bottomrule
\end{tabular}
\end{table}

\begin{figure*}[t]
    \centering
    \includegraphics[width=0.5\textwidth]
{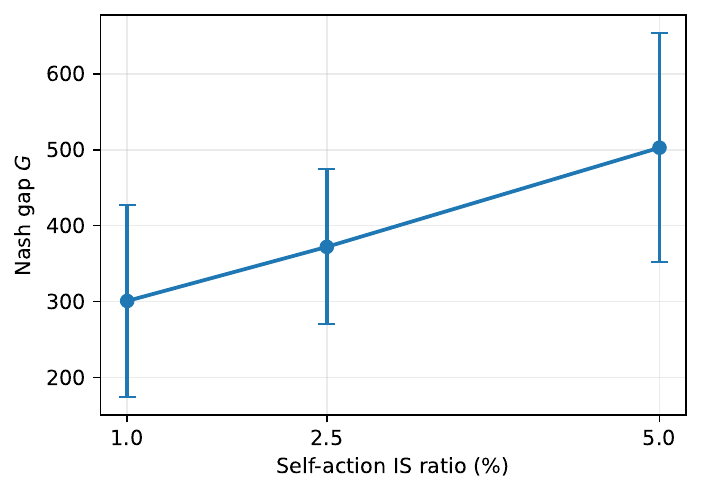}
    \caption{\textbf{Self-action IS ratio sensitivity.} Nash gap on F2 (m=n=3) as the self-action IS ratio varies. Lower is better.}
    \label{fig:self-action-is-hyperparamter-study}
\end{figure*}

\newpage 

\section{Computing Infrastructure}

Experiments were run entirely on CPU hardware, using Intel Xeon Platinum 8268 nodes (48 cores, 192 GB RAM) and 32-vCPU cloud instances.

\end{document}